\documentclass[final,12pt]{colt2025} 
\makeatletter
\let\Ginclude@graphics\@org@Ginclude@graphics
\makeatother

\usepackage{amssymb}
\usepackage{dsfont}
\usepackage{amsmath}
\usepackage{mathtools}
\usepackage{mathtools}
\usepackage{pb-diagram}
\usepackage{comment}
\usepackage{float}
\usepackage{wrapfig}
\usepackage{caption}
\usepackage{graphicx}

\newtheorem{setting}{Setting}

\usepackage{xcolor}

\definecolor{darkcyan}{rgb}{0.0, 0.55, 0.55}
\definecolor{MidnightBlue}{RGB}{25,25,112}
\definecolor{MidnightBlueComplementingGreen}{RGB}{25,112,25}
\definecolor{MidnightBlueComplementingPurple}{RGB}{112,25,112}
\definecolor{MidnightBlueComplementingRed}{RGB}{112,25,69}
\definecolor{WowColor}{rgb}{.75,0,.75}
\definecolor{MildlyAlarming}{rgb}{0.85,0.25,0.1}
\definecolor{SubtleColor}{rgb}{0,0,.50}
\definecolor{antiquefuchsia}{rgb}{0.57, 0.36, 0.51}
\definecolor{fashionfuchsia}{rgb}{0.96, 0.0, 0.63}
\definecolor{jade}{rgb}{0.0, 0.66, 0.42}
\definecolor{caribbeangreen}{rgb}{0.0, 0.8, 0.6}
\definecolor{aquamarine}{rgb}{0.5, 0.8, 0.85}
\definecolor{lightseagreen}{rgb}{0.13, 0.7, 0.67}
\definecolor{darkgreen}{rgb}{0.0, 0.2, 0.13}
\definecolor{darkspringgreen}{rgb}{0.09, 0.45, 0.27}
\definecolor{attentioncolor}{RGB}{152,90,81}
\definecolor{burgred}{RGB}{40,3,22}
\definecolor{AnnieGreen}{RGB}{17,123,92}
\definecolor{Turquoise}{RGB}{64,224,208}

\definecolor{darkjade}{RGB}{0,122,84}
\definecolor{Window1}{RGB}{92,150,31}%
    \definecolor{Window1dark}{RGB}{41,67,13}%
\definecolor{Window2}{RGB}{255,168,28}
    \definecolor{Window2dark}{RGB}{114,75,12}
\definecolor{Window3}{RGB}{255,96,33}
    \definecolor{Window3dark}{RGB}{97,36,12}
\definecolor{InputColor}{RGB}{20,255,177}
    \definecolor{InputColorlight}{RGB}{222,237,229}

\definecolor{RedAlizarin}{rgb}{0.82, 0.1, 0.26}

\definecolor{darkcerulean}{rgb}{0.03, 0.27, 0.49}
    \definecolor{smokyblack}{rgb}{0.06, 0.05, 0.03}
    \definecolor{warmblack}{rgb}{0.0, 0.26, 0.26}
    \definecolor{cobalt}{rgb}{0.0, 0.28, 0.67}
    \definecolor{darkcobalt}{rgb}{0.1, 0.38, 0.77}

\definecolor{MidnightBlue}{RGB}{25,25,112}
\definecolor{MidnightBlueComplementingGreen}{RGB}{25,112,25}
\definecolor{MidnightBlueComplementingPurple}{RGB}{112,25,112}
\definecolor{MidnightBlueComplementingRed}{RGB}{112,25,69}
\definecolor{coolblack}{rgb}{0.0, 0.18, 0.39}

\definecolor{deepjunglegreen}{rgb}{0.0, 0.29, 0.29}
\definecolor{applegreen}{rgb}{0.55, 0.71, 0.0}
\definecolor{WowColor}{rgb}{.75,0,.75}
\definecolor{MildlyAlarming}{rgb}{0.85,0.25,0.1}
\definecolor{SubtleColor}{rgb}{0,0,.50}
\definecolor{SubtleColor2}{rgb}{0.6,0.21,.50}

\definecolor{lasallegreen}{rgb}{0.03, 0.47, 0.19}
\newcounter{margincounter}

\NewDocumentCommand{\AK}{moo}{
    \IfValueF{#2}{
                        {{\scriptsize
                            \textcolor{deepjunglegreen}{
                                \textbf{A:}
                                {#1}
                            }
                        }}
        }
    \IfValueT{#2}{\IfValueF{#3}{
                        {{\scriptsize
                            \textcolor{deepjunglegreen}{
                            \hfill\\
                                \noindent 
                                \textbf{A:}
                                \textit{{#1}}
                            \hfill\\
                            }
                        }}
        }}
    \IfValueT{#3}{
                        \marginnote{{\scriptsize
                            \textcolor{deepjunglegreen}{ 
                            \textbf{A:}
                            \textit{{#1}}
                            }
                        }}
        }
                    }

\newcommand{\Takashi}[1]{{\color{blue}[\textit{#1} Takashi]}}

\newcommand{\eqdef}{\ensuremath{\stackrel{\mbox{\upshape\tiny def.}}{=}}}

\newcommand{\xxx}{{\mathcal{X}}}

\newcommand\numberthis{\addtocounter{equation}{1}\tag{\theequation}}

\usepackage{cleveref}
\newcounter{termcounter}
\renewcommand{\thetermcounter}{\Roman{termcounter}}
\crefname{term}{term}{terms}
\creflabelformat{term}{#2\textup{(#1)}#3}

\makeatletter
\def\term{\@ifnextchar[\term@optarg\term@noarg}
\def\term@optarg[#1]#2{%
  \textup{#1}%
  \def\@currentlabel{#1}%
  \def\cref@currentlabel{[][2147483647][]#1}%
  \cref@label[term]{#2}}
\def\term@noarg#1{%
  \refstepcounter{termcounter}%
  \textup{(\thetermcounter)}%
  \cref@label[term]{#1}}
\makeatother

\title[Is In-Context Universality Enough?]{Is In-Context Universality Enough? \hfill\\
MLPs are Also Universal In-Context}
\usepackage{times}

\coltauthor{
 \Name{Anastasis Kratsios} \Email{kratsioa@mcmaster.ca}\\
 \addr The Vector Institute and McMaster University
 \AND
 \Name{Takashi Furuya}
 \Email{takashi.furuya0101@gmail.com}\\
 \addr Shimane University}

\begin{document}

\maketitle

\begin{abstract}
The success of transformers is often linked to their ability to perform \textit{in-context} learning. Recent work shows that transformers are universal in context, capable of approximating any real-valued continuous function of a context (a probability measure over $\mathcal{X}\subseteq \mathbb{R}^d$) and a query $x\in \mathcal{X}$. This raises the question: \textit{Does in-context universality explain their advantage over classical models}?
We answer this in the \textbf{negative} by proving that MLPs with trainable activation functions are also \textit{universal in-context}. This suggests the transformer’s success is likely due to other factors like inductive bias or training stability.

\end{abstract}


\section{Introduction}
\label{s:Introduction}

The undeniable success of transformers is attributed to their ability to learn \textit{in-context}, unlike traditional multilayer perceptions (MLPs). This means that transformers can process sequences of tokens and predict the next relevant token. They often exhibit in-context learning (ICL) when trained on large, diverse datasets. This means that given a short sequence of input-output pairs (a prompt) from a specific task, the model can make predictions on new examples without updating its parameters.  

There are three primary pillars in which ICL can be studied: approximation theoretic, statistical, and optimization lenses. We focus on the former of these lenses by contrasting the approximation capacity of the transformer with the classical MLP model by asking 
\begin{equation}
\label{eq:Main_Q}
\tag{Q}
\mbox{\textit{``Does the transformer have an approximation-theoretic advantage over the MLP?''}}
\end{equation}
Due to the work of~\cite{hornik1989multilayer,yarotsky18a_verydeep_COLT,suzukiadaptivity,bolcskei2019optimal,kidger2020universal,kratsios2022universal,ZuoweiHaizhaoZhang_2022_JMPA}, 
and several others, it is by now well-known that MLPs are universal approximators in the classical ``out-of-context'' sense; i.e., meaning that any continuous functions from $\mathbb{R}^d$ to $\mathbb{R}^D$ can be uniformly approximated on compact sets to arbitrary precision by MLPs with enough neurons.  
These classical ``out-of-context'' universal approximation guarantees have since been established for the transformer model~\cite{kim2024transformers,fang2022attention} with matching optimal rates, implying that the transformer is at least as expressive as the MLP ``out-of-context''. 

More recently~\cite{petrov2024universal} showed that recurrent transformers can approximate any function in-context by leveraging prompt engineering.  It was subsequently established by~\cite{furuya2024transformers} that transformers are universal approximators \textit{in-context}, meaning that they can approximate any function which continuously maps \textit{context} and \textit{queries} to predictions uniformly on compact sets to arbitrary precision; again given enough neurons.  These results \textit{suggest} that the transformer may indeed have an advantage in expressivity over the vintage MLP model, hinting that~\eqref{eq:Main_Q} can be hoped to be answered positively since the latter is currently \textit{not known} to be universal in-context.  

Our paper is in the \textbf{negative} direction of~\eqref{eq:Main_Q}.  Our main result (\textit{Theorem}~\ref{thrm:Main__SimpleVersion}) matches the \textit{in-context universality} of the MLP model in the setting of \textit{permutation invariant contexts (PICs)} of~\cite{castin2024smooth,furuya2024transformers}.  We conclude that if the transformer's superior empirical performance over the MLP model cannot be explained by \textit{in-context universality}.  
This suggests that the empirically well-documented advantage of transformers over MLPs must stem from a statistical or optimization phenomenon unique to transformers rather than from in-context universality.

Our main result is complemented by Corollary~\ref{cor:Main_TransformerVersion__SimpleVersion}, which a quantitative version of~\cite{furuya2024transformers}, showing that for any target function, compact set of PICs, and approximation error, a transformer with multiple attention heads per block.  Thus, the in-context approximation power of the transformer seems to match that of the MLP. 

\subsection{Secondary Contributions}
Our second main result (Corollary~\ref{cor:Main_TransformerVersion__SimpleVersion}) is deduced from our main result for ReLU MLPs, using our \textit{transformerification} procedure (Proposition~\ref{prop:transformerification__SparseVersion}), which converts any MLP to a multi-head transformer which exactly preserves its depth and width while only doubling its trainable (non-zero) parameters (thus has the same order of trainable parameters) and with a fixed number of attention heads per block.  This type of ``conversion'' map was also recently obtained to prove quantitative universal approximation guarantees for convolutional neural networks in~\cite{petersen2020equivalence} and spiking neural networks in~\cite{singh2023expressivity}.

A key step in our analysis shows that ReLU MLPs can \textit{exactly} implement the $1$-Wasserstein distance on a broad class of finite probability measures, containing all empirical measures (Proposition~\ref{prop:Computation_W1__relative_verison}).  This auxiliary result can be of independent interest to the neural optimal transport community, e.g.~\cite{korotin2019wasserstein,korotinneural,gazdieva2024robust}.

\subsection{Related Literature}
\label{s:Introduction__ss:RelatedLiterature}

\paragraph{Permutation-Invariant Contexts and their Geometry}
Transformer models generally cannot inherently detect the order of input tokens without using \textit{positional encodings}~\cite{vaswani2017attention,chuconditional,li2021learnable}. These encodings capture the sequence in which tokens appear, addressing the transformer's invariance to row permutations in the input context matrix. 
Since we are studying the transformer architecture and not a specific positional encoding scheme; then, we will adopt the \textit{permutation-invariant} formulation of context of~\cite{castin2024smooth,furuya2024transformers}, rather than the \textit{sequential} formulation of~\cite{garg2022can,zhang2024trained,akyurek2022learning,von2023transformers}, which does not capture the transformer's permutation invariance.

In this paper, we work with a refined version of the permutation-invariant setting introduced by~\cite{furuya2024transformers}. Their framework formalizes contexts as probability measures on a space of tokens within a \textit{dictionary} $\mathcal{X}\subseteq \mathbb{R}^n$. The motivation is that any finite set of tokens $x_1,\dots,x_N\in \mathcal{X}$ can be represented as an empirical measure $\mu=\sum_{n=1}^N w_n \delta_{x_n}$, where $0\le w_1,\dots,w_N\le 1$ (summing to $1$), reflecting the \textit{relative frequency} of each token in the \textit{permutation-invariant context (PIC)} $\mu$, which is unaffected by the indexing of these tokens. While their permutation-invariant setting is well-suited for real-world transformers, it overlooks the fact that real-world \textit{context windows} ($C$), even when tokens are drawn from an infinite dictionary $\mathcal{X}$, still impose constraints.

Our analysis rests on a quantitative refinement of the mathematical idealization in the setting of~\cite{furuya2024transformers}, which assumed that context can be arbitrarily large since any real-world LLM has a finite context window.  Instead, our analysis operates on the realistic subspace $\mathcal{P}_{C,N}(\mathcal{X})$ (formalized in Section~\ref{s:Prelims__ss:PICL}) comprised of probability measures of the form $\sum_{n=1}^N\, w_n\delta_{x_n}$ where the weights $\{w_n\}_{n=1}^N$ are all positive and \textit{divisible by the context window $C$}.  Thus, each distinctly observed token $x_n$ cannot be observed more times than the context window $C$ allows. I.e.\ we prohibit mathematical artifacts such as a relative frequency of $1/\sqrt{2}$.  Under this realism restriction, we can identify any PIC in $\mu\in \mathcal{P}_{C,N}(\mathbb{R}^d)$ an equivalence class $X^{\mu}$ of $N\times (d+1)$ real-matrix
\begin{equation}
\label{eq:Identification}
            \mu
        =
            \sum_{n=1}^N\, w_n\delta_{x_n} 
    \leftrightarrow
            X^{\mu}
        \eqdef 
            \big[(x_n,w_n)_{n=1}^N\big]
\end{equation} 
here, $[A]$ denotes the equivalence class of matrices formed by permuting the rows of the $N \times (d+1)$ matrix $A$. We equip $\mathcal{P}_{C,N}(\mathbb{R}^d)$ with the $1$-Wasserstein metric, which, when restricted to the space of empirical measures $\mathcal{P}_{1,N}(\mathbb{R}^d)$, is shown to be equivalent to the natural quotient metric (see~\cite{BridsonHaefliger_1999NPCBook}) on the corresponding subspace of matrices, quotiented by row permutations. Details are provided in Section~\ref{s:Prelims__ss:PICL}.

\paragraph{Approximation Guarantees}
We highlight that there are alternative, weaker, approximation guarantees for transformers than the one in~\cite{furuya2023globally}, some of which consider contexts; e.g.~\cite{petrov2024universal}, without permutation-invariance, and some which are context-free; e.g.~\cite{kim2024transformers,fang2022attention}. The conclusion remains the same when juxtaposed against Theorem~\ref{thrm:Main__SimpleVersion}; in-context universality is not enough to explain the advantage of the transformer model over the classical MLP model.

Although not directly related to our findings, there is a body of work that examines the expressivity of transformers when applied to discrete token sets as formal systems~\cite{chiang2023tighter,merrill2023expresssive,strobl2024formal,olsson2022context}. Another line of research also explores how positional encoding affects transformer expressivity~\cite{luo2022your}. Nevertheless, the transformer's advantage over the classical MLP model is not identified. We also mention results explaining the ability of transformers to represent specific structures, such as Kalman filtering updates relying on a Nadaraya–Watson kernel density estimator~\cite{goel2024can} and its ability to satisfy certain constraints~\cite{zamanlooy2021trans}.

\paragraph{Statistical Guarantees}
We also mention the growing body of statistical guarantees for in-context learners, such as transformers.  These study the training dynamics of transformers with linear attention mechanisms~\cite{chen2024training,lu2024asymptotic,zhang2024context,siyu2024training,kim2024transformersdyna} and non-asymptotic single-step variants thereof~\cite{duraisamy2024finite}, the minimax statistical optimality of pre-trained transformers trained in-context~\cite{kim2024transformers}, guarantees for transformers trained with non-i.i.d.\ time-series data~\cite{limmer2024reality}, PAC-Bayesian guarantees for transformers~\cite{mitarchuk2024length}, their infinite-width limits~\cite{zhang2024trained} in an NTK fashion~\cite{jacot2018neural}, statistical guarantees for in-context classification using transformers~\cite{reddy2023mechanistic}, and several other works examining the statistical foundations of in-context learning; e.g.~\cite{akyurek2022learning,li2023transformers,bai2024transformers}.  We mention the recent line of work studying the efficiency of trains of thoughts generated by transformers trained in context~\cite{kim2024transformersTOT}.

\section{Preliminaries}
\label{s:Prelims}
This section contains the necessary background for the formulation of our main results. 


\subsection{Permutation-Invariant In-Context Learners}
\label{s:Prelims__ss:PICL}
We first review and add to the notions of \textit{permutation invariant} in context-learners considered in~\cite{furuya2024transformers}.  This relies on the introduction of some tools from probability theory, specifically from \textit{optimal transport} and their refinements use ideas from \textit{metric geometry}; both of which are introduced now.

\paragraph{Permutation-Invariant Context via Probability Measures}

Given a (non-empty) measurable subset $\xxx$ of $\mathbb{R}^d$, for some $d\in \mathbb{N}_+$, we use  $\mathcal P(\xxx)$ to denote the set of all Borel probability measures on $\xxx$ which is then equipped with the topology of weak convergence of measures.
When $1 \le p < \infty$, we denote by $\mathcal P_p(\xxx)$ the subset of probabilities that finitely integrate $x \mapsto \|x-x_0\|^p$ for some (and thus for any) $x_0 \in \xxx$.
Similarly, we equip $\mathcal P_p(\xxx)$ with the Wasserstein $p$-distance $\mathcal W_p$, that is, for $\mu,\nu \in \mathcal P_p(\xxx)$, the metric defined by
$
        \mathcal W_p(\mu,\nu)^p 
    \eqdef
        \inf_{\pi \in \operatorname{Cpl}(\mu,\nu)} 
            \int 
                d_\xxx(x,y)^p
                \,
                \pi(dx,dy)
            ,
$
where $\Pi(\mu,\nu) \eqdef \{ \pi \in \mathcal P(\xxx \times \xxx) \colon \pi \text{ has first marginal } \mu, \text{ second marginal }\nu \}$.  Elements of $\Pi(\mu,\nu)$ are called couplings, or transport plans, between the measures $\mu$ and $\nu$.  
When $p=1$ the Kantorovich-Rubinstein duality allows us to re-express and extend the definition of $\mathcal{W}_1$ to the class of finite signed (Borel) measures on $\mathcal{X}$ via
\[
\mathcal{W}(\mu,\nu)\eqdef 
\sup_{f\in \operatorname{Lip}(\mathcal{X},1)} \int f(x) d (\mu -\nu)(x)
\]
where $\operatorname{Lip}(\mathcal{X},1)$ denotes the set of $1$-Lipschitz functions $f$ on $\mathcal{X}$ with $\|f\|_{\infty}+L_f\le 1$, where $L_f$ is the optimal Lipschitz constant of $f$ and $\|f\|_{\infty}\eqdef \sup_{x\in \mathcal{X}}\,|f(x)|$.   This extension to signed finite measures is typical in the non-linear theory of Banach space theory~\cite{godefroy2015survey,weaver2018lipschitz,ambrosio2020linear} and has applications in deep learning~\cite{von2004distance,kratsios2023approximation,cuchiero2023global}.

Though infinite-length contexts are considered in some parts of the literature, contexts of a finite \textit{but possibly very large} are both more representative of practical use-cases of ICL. They are more amenable to precise quantitative analysis.  Thus, in this paper, we fix: a finite \textit{context window} $C\in \mathbb{N}_+$, a \textit{token number} $N\in \mathbb{N}_+$ of contextual data, and a dictionary set $\mathcal{X}\subseteq \mathbb{R}^d$ with at-least $N$ distinct points.
We define the contextualized simplex $\Delta_{C,N}$
\begin{equation}
\label{eq:contexualized_simplex}
        \Delta_{C,N}
    \eqdef 
        \biggl\{
                w\in [\{c/C\}_{c=1}^C]^N
            :
            \, 
                \sum_{n=1}^N\, w_n =1
        \biggr\}
.
\end{equation}
\begin{figure}[htp!]
    \centering
    \begin{minipage}{.2\textwidth}
        \centering
        \includegraphics[width=\linewidth]{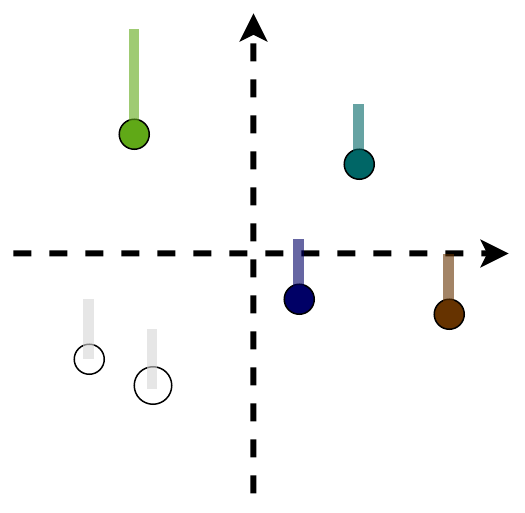}
        \caption{As \textit{Measure} in $\mathcal{P}_{7,6}(\mathbb{R}^2)$}
        \label{fig:Token}
    \end{minipage}%
~
    \begin{minipage}{0.5\textwidth}
        \centering
        \includegraphics[width=.75\linewidth]{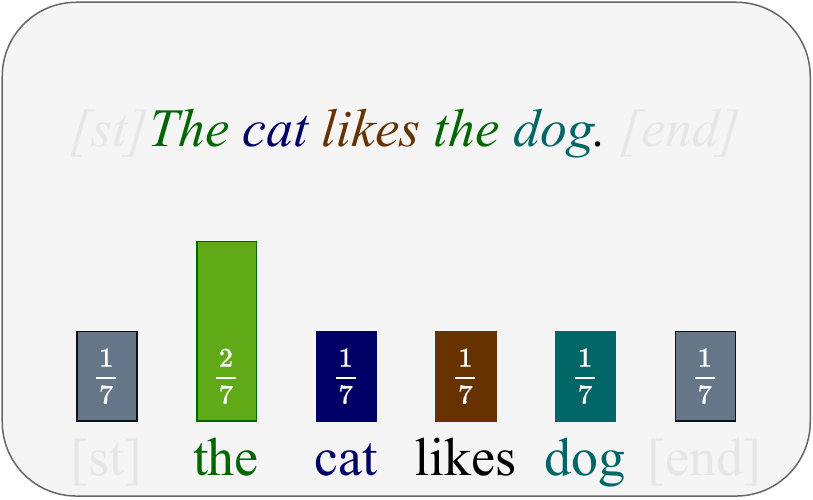}
        \caption{Context as \textit{Tokens}.}
        \label{fig:Language}
    \end{minipage}    
~   
    \begin{minipage}{0.2\textwidth}
        \centering
        \includegraphics[width=\linewidth]{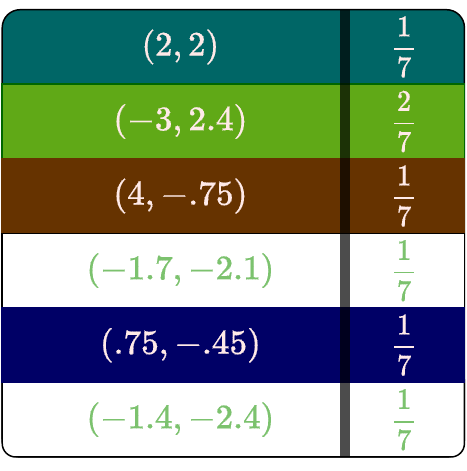}
        \caption{As \textit{Matrix} in $\operatorname{Mat}_{7}^{2,6}/\sim$}
        \label{fig:Matrix}
    \end{minipage}
\caption{\textbf{Permutation-Invariant Context (Definition~\ref{defn:PICs}):}
    In practice, context lengths consist of a finite number of tokens ($N$), each of which is repeated at most a finite number of times; counting repetitions/multiplicities, the total number of tokens is ($C$) finite and fixed.  
    In this representation, the order in which the tokens ``the'' appeared are disregarded.
    \hfill\\
    In this illustration, there are $N=6$ tokens: ``the'', ``cat'', ``likes'', ``dog'', as well as the punctuation tokens indicating the start ``[st]'', and end ``[end]'' of the sentence in the illustrated prompt.  Only ``the'' is repeated twice, and the context window $C=7$.}
    \label{fig:Realistic_Contexts}
\end{figure}
In short, $\Delta_{C,N}$ is a discretization of the relative interior of the $N$-simplex.
The space PICs consists of all probability measures in $\mathcal{P}_1(\mathbb{R}^d)$ counting mass on at-most $N$ distinct tokens where each token appears at-most $C$ times; formalized as follows.

\begin{definition}[Permutation-Invariant Context]
\label{defn:PICs}
Let $C,N,d\in \mathbb{N}_+$ and $\mathcal{X}\subseteq \mathbb{R}^d$ be non-empty.  
The space $\mathcal{P}_{C,N}(\mathcal{X})$ of permutation-invariant contexts (PICs) on $\mathcal{X}$ is the metric subspace of $\mathcal{P}_1(\mathcal{X})$ consisting of all $\mu=\sum_{n=1}^N\, w_n\delta_{x_n}$ satisfying
\begin{equation}
\label{eq:PIC}
    \underbrace{
        x_1,\dots,x_N \in  \mathcal{X}
    }_{\text{Domain}}
    \mbox{, }
        \underbrace{
            x_1\neq \dots \neq x_N
    ,
        }_{\text{No token duplication}}
    \mbox{ and }  
        \underbrace{
        w\in \Delta_{C,N} 
        }_{\text{Context limit}}
.
\end{equation}
\end{definition}

Probability measures provide an intuitive mechanism by which we may interpret the frequency of any context in a given context window while ignoring ad-hoc orders due to the permutation invariance of the points $x_1,\dots,x_N$ representing any $\sum_{n=1}^N\,w_n\,\delta_{x_n}\in \mathcal{P}_{C,N}(\mathcal{X})$.  Nevertheless, most deep learning models such as transformers, GNNs, or MLPs do not act on measures but on matrices.  To bridge the link between our mathematically natural formulation of PIC in~\eqref{eq:PIC} and the input/outputs of these models, we first identify PICs with certain \textit{equivalence classes of matrices}.

\paragraph{Matrix Representations of PICs}
Let $\operatorname{Mat}^{d,N}_C$ denote the set of $N\times (d+1)$ matrices $X$ whose rows are given by $(x_1,w_1),\dots,(x_N,w_N)\in \mathbb{R}^{d+1}$ and such that $
    (w_1,\dots,w_N)\in \Delta_{C,N}
$.  

We \textit{quotient} $\operatorname{Mat}^{d,N}_C$ by the equivalence relation $X\sim Y$ if there is a $N\times N$ permutation matrix $\Pi$ for which
$
    Y=\Pi X
$
.
Our analysis relies on the map 
$\Phi: 
    \mathcal{P}_{C,N}(\mathcal{X})
    \rightarrow 
    \operatorname{Mat}^{d,N}_C/\sim$
\begin{equation}
\label{eq:identification}
\begin{aligned}
    \sum_{n=1}^N\,
        w_n
        \delta_{x_n}
    & \mapsto 
        \begin{bmatrix}
            (x_1,w_1)\\
            \vdots\\
            (x_N,w_N)
        \end{bmatrix}
\end{aligned}
\end{equation}
which is easily verified to be a bijection between $\mathcal{P}_{C,N}(\mathcal{X})$ and $\operatorname{Mat}^{d,N}_C/\sim$.  
The identification in~\eqref{eq:identification} allows us to put a \textit{metric structure} on $\operatorname{Mat}_C^{d,N}\sim$ which is identical to the $1$-Wasserstein distance on $\mathcal{P}_{C,N}(\mathcal{X})$ inherited from $\mathcal{P}_1(\mathbb{R}^d)$.  
This metric, which simply denote by $\mathcal{W}:\operatorname{Mat}^{d,N}_C/\sim\times \operatorname{Mat}^{d,N}_C/\sim\to [0,\infty)$, sends any pair of $
X^{\mu}\eqdef 
[(x_n,w_n)_{n=1}^N]$ and $
Y^{\nu}\eqdef 
[(y_n,u_n)_{n=1}^N]$ of equivalence classes of matrices in $\operatorname{Mat}^{d,N}_C/\sim$ to
\begin{equation}
\label{eq:W1_onMat}
    \mathcal{W}\big(
            X^{\mu}
        ,
            Y^{\nu}
    \big)
\eqdef
    \mathcal{W}\biggl(
            \Phi^{-1}\big(
                X^{\mu}
            \big)
        ,
            \Phi^{-1}\big(
                Y^{\nu}
            \big)
    \biggr)
=
    \mathcal{W}_1\Biggl(
        \sum_{n=1}^N\,
            w_n
            \delta_{x_n}
    ,
        \sum_{n=1}^N\,
            u_n
            \delta_{y_n}
    \Biggr)
.
\end{equation}
Thus, by construction $\operatorname{Mat}_C^{d,N}$ metrized by $\mathcal{W}$ is \textit{isometric} (i.e.\ indistinguishable as a metric space) from $\mathcal{P}_{C,N}(\mathcal{X})$ equipped with the $1$-Wasserstein distance {on signed finite measures, as defined in~\cite{villani2009optimal} on probability measures on~\cite{ambrosio2020linear} on finite measures. 

\paragraph{A Geometric Interpretation of the Metric $\mathcal{W}$}
Before moving on, we further motivate the choice of metric $\mathcal{W}$ by considering the special case where the context window $C$ exactly equals to the number of tokens $N$.  In this case, each weight $w_1=\dots=w_C=\frac1{N}$ in~\eqref{eq:PIC} and $\mathcal{P}_{N,N}(\mathcal{X})$ consists only of empirical measures with distinct support points in $\mathcal{X}$.  Whence, $\operatorname{Mat}_C^{N,N}$ is identifiable with the space of equivalence classes of $N\times d$ matrices with distinct rows, up to row permutation and the identification in~\eqref{eq:identification} simplifies to
\begin{equation}
\label{eq:identification__simplified}
\smash{
        \sum_{n=1}^N\, \frac1{N}\delta_{x_n}
    \leftrightarrow
        \big[
            (x_n)_{n=1}^N
        \big]
.
}
\end{equation}
Now, $\operatorname{Mat}^{d,N}_N/\sim$ is the quotient space of the space $\operatorname{Mat}^{d,N}_N$, namely, the space of $d\times N$ matrices with distinct rows, by (row) permutations.   
We, momentarily, equip $\operatorname{Mat}_N^{d,N}$ with the \textit{Fr\"{o}benius norm} $\|\cdot\|_2$.
Since the space of $N\times N$ permutation matrices $\mathbb{S}^N$ is a finite group acting \textit{by isometries} on $\operatorname{Mat}_N^{d,N}$; i.e.\ for each $\Pi\in \operatorname{S}^N$ and every $X\in \operatorname{Mat}_N^{d,N}$ we have
\[
    \|X\|_F = \|\Pi\cdot X\|_2
\]
then~\cite[Exercise 8.4 (3) - page 132]{BridsonHaefliger_1999NPCBook} guarantees that the (natural) quotient topology on $\operatorname{Mat}_N^{d,N}/\sim$ is metrized by the metric $\operatorname{dist}:\operatorname{Mat}_N^{d,N}/\sim\times \operatorname{Mat}_N^{d,N}/\sim\to [0,\infty)$ defined for any 
$[X],[Y]$ in $\operatorname{Mat}^{d,N}_C/\sim$ by
$
        \operatorname{dist}([X],[Y])
    \eqdef 
        \inf_{\Pi\in \mathbf{S}^N}
        \,
        \|\Pi X - Y\|_2
$
.
The geometric motivation for $\mathcal{W}$ is drawn from the fact that $\mathcal{W}\asymp \operatorname{dist}$ on $\operatorname{Mat}_N^{d,N}/\sim$.

\begin{proposition}[{Equivalence of $\mathcal{W}$ and the Natural Quotient Metric on PICs}]
\label{prop:Identifcation}
Let $N,d\in \mathbb{N}_+$ and let $\mathcal{X}\subseteq \mathbb{R}^d$ be non-empty.  Then, there are absolute constant $0<c\le C$ such that: for each $[X],[Y]\in \mathcal{P}_N^{N,d}(\mathcal{X})$ we have
\[
    c\mathcal{W}([X],[Y]) \le \operatorname{dist}([X],[Y]) \le C\mathcal{W}([X],[Y])
.
\]
In particular, the map $\Phi:(\mathcal{P}_{N,N}(\mathcal{X}),\mathcal{W}_1)\to (\operatorname{Mat}^{d,N}_N/\sim,\operatorname{dist})$ is a homomorphism.
\hfill\\
Furthermore, $\mathcal{W}$ metrizes the natural quotient topology on $\operatorname{Mat}_N^{d,N}/\sim$.
\end{proposition}
Proposition~\ref{prop:Identifcation} shows that $\mathcal{W}$ is well-motivated in that it is equivalent to the most natural metric structure on the quotient space $\operatorname{Mat}_C^{N,d}/\sim$, at-least, in the special case where $N=C$. 

\subsubsection{The Size of Sets of PICs}
\label{s:Prelims__ss:Background__ss:size_of_compacta}
This section quantifies the \textit{size} of compact sets of PICs.  The reader not interested in quantitative universal approximation guarantees is encouraged to skip it.


\paragraph{Basic Definitions}
Fix a compact set $\mathcal{K}$ of contexts in $\mathcal{P}_{C,N}(\mathcal{X})\times \mathcal{X}$.  
We denote the ball at any point $x\in X$ of radius $r>0$ by $B((\mu,x),r)\eqdef\{(\nu,z)\in \mathcal{K}:\,\mathcal{W}_1(\mu,x)+\|z-x\|_1<r\}$.  We say that $(\mu,x),(\nu,y)\in \mathcal{K}$ are $r$-separated, for some $r>0$, if $\mathcal{W}_1(\mu,\nu)+\|x-y\|_1\ge r$.

\paragraph{Metric Notions of Dimension}
We say $\mathcal{K}$ is $q$-doubling $0\le q<\infty$ if there is a constant $C>0$ such that: for each pair of radii $0 < r \leq R \leq \text{diam}(\mathcal{K})$ and every subset $A \subseteq \mathcal{K}$ of diameter at-most $R$, there is no more than $C\left(\frac{R}{r}\right)^q$ point in $A$ which are $r$-separated. The minimum (infimum) such number $q\ge 0$ is called the \textit{doubling (Assouad) dimension} of $\mathcal{K}$.


Every such compact doubling metric space $\mathcal{K}$ carries a (Borel) measure $\mathbb{P}$ which ascribes such that, there is a ``doubling constant'' $C_{\mathbb{P}}>0$ with the property that: for each $(\mu,x)\in \mathcal{K}$ and each $r>0$ we have 
$
\mathbb{P}\big(B(x,r)\big)
\le C_{\mathbb{P}}
$
,
see~\cite[Theorem 3.16]{cutler1995density} and~\cite{MR1443161}.  Examples include the uniform measure on $[0,1]^q$ and the normalized Riemannian measure on a compact Riemannian manifold.  Unlike these familiar measures, which are compatible with their metric structure, some general metric spaces admit pathological metric measures which are incompatible with their metric dimension; see~\cite [Example 3.1]{PackingAlfors_MatematisceZeitschrift_2010}.  To avoid these pathologies, we focus on \textit{Ahlfors $q$-regular} measures $\mathbb{P}$ on $\mathcal{K}$
which attribute  $\Omega(r^q)$ mass to any ball of radius $r>0$ in $\mathcal{K}$; i.e.\ there are constants $0<c\le C$ such that:  for $x\in X$ and every $r>0$ 
\begin{equation}
\label{eq:Ahlors}
\smash{
        cr^q
    \le 
        \mathbb{P}\big(
            B(x,r)
        \big)
    \le
        Cr^q
.
}
\end{equation}
We summarize the doubling and Ahlfors regularity requirements as follows.
\begin{definition}[$q$-Dimensional PIC]
\label{defn:dimension}
Let $q>0$.
A compact subset $\mathcal{K}\subseteq \mathcal{P}_{C,N}(\mathcal{X})\times \mathcal{X}$ of PICs equipped with a Borel probability measure $\mathbb{P}$ is said to be $q$-dimensional, its doubling dimension is $q$ and $\mathbb{P}$ is Ahlfors $q$-regular.
\end{definition}


The compatibility between the metric structure on $\mathcal{K}$ and the measure $\mathbb{P}$ given by the Ahlfors condition in condition~\eqref{eq:Ahlors} allows us to intrinsically quantify the size of subsets of $\mathcal{K}$ on which our approximation may fail.  These subsets, illustrated later as the small trifling region in Figure~\ref{fig:TrifflingRegions}, are non-Euclidean versions of the trifling regions in classical optimal universal approximation guarantees for MLPs; as in~\cite{ZuoweiHaizhaoZhang_2022_JMPA}.

\subsection{Deep Learning and Transformers}
\label{s:Background__ss:DeepLearning}

\begin{minipage}{0.45\textwidth}
\centering
\includegraphics[width=.25\linewidth]{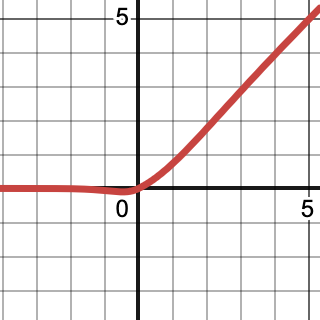}
\captionof{figure}{The activation function in~\eqref{eq:activation}.} 
\label{fig:aftive}
\end{minipage}%
\begin{minipage}{0.45\textwidth}
\begin{equation}
\label{eq:activation}
        \sigma_{\theta}(x)
    \eqdef 
        \begin{cases}
            \alpha_1 x^p & \mbox{ if } x\ge 0\\
            \alpha_2 x & \mbox{ if } x < 0
        \end{cases}
.
\end{equation}
\end{minipage}

\paragraph{Trainable Activation Functions}
As most modern deep learning implementations, such as transformers using trainable Swish~\cite{ramachandran2017searching} or GLU activation variants~\cite{shazeer2020glu}, KANs~\cite{liu2024kan,ismailov2024addressing} activated by trainable B-splines and SiLUs, or neural snowflakes which implement trainable fractal activation functions~\cite{de2023neural,borde2024neural}, have shifted towards trainable activation functions. 
We follow suit by using trainable activation functions in our MLPs, which can implement the Swish and leaky ReLUs activation function, skip connections blocks, and $\operatorname{ReQU}(x)=\max\{0,x\}^2$, see Appendix~\ref{a:Activ} for details
and defined by

\paragraph{The MLP Model}
In what follows, we will often make use of fully-connected MLPs with activation function $\sigma$ as in~\eqref{eq:activation}.  For any $d,D\in \mathbb{N}_+$, we recall that an MLP with activation function $\sigma$ is a map $f:\mathbb{R}^d\to\mathbb{R}$ with iterative representation
\begin{equation}
\label{eq:representation_MLP}
    \begin{aligned}
    f(x)
    \eqdef 
    A^{(J)} X^{(J)}
    ,\quad
    \mathbf{X}^{(j+1)} 
    \eqdef 
    A^{(j)} \sigma_{\bar{\alpha}^{(j)}}\bullet(
        \mathbf{X}^{(j)}
            +
        b^{(j)})
    %
    ,\quad
    \mathbf{X}^{(0)} 
    \eqdef \mathbf{X}
    \end{aligned}
\end{equation}
where $j$ runs from $0$ to $J-1$ and for each $j\in [J]$ $A^{(j)}$ is a $d_{j+1}\times d_j$ matrix with $d_0=d$ and $d_{J+1}=D$.  Here, $J$ is the depth of the transformer and $\max_{j=1,\dots,J-1}\,d_j$ is its width.

\paragraph{Transformers with Multi-Head Attention}
Building on the notation in~\eqref{eq:representation_MLP}, we formalize a transformer model.  We thus first formalize a single attention head of ~\cite{bahdanau2014neural}, with temperature parameter $\lambda>0$, as a map sending any PIC $\mu \sim \mathbf{X}\in \mathbb{R}^{N\times d}$ to the following probability measure in $\mathcal{P}_{C,N}(\mathcal{Y})$
\[
    \operatorname{Attn}(x,\mathbf{X}|Q,K,V,\lambda)
    \eqdef
        \sum_{n=1}^N\,
            w_n\,
            \frac{
                e^{\lambda \langle Q x, K\mathbf{X}_n\rangle/\sqrt{d}}
            }{
                \sum_{m=1}^N\,
                    e^{\lambda \langle Q x, K\, \mathbf{X}_m \rangle/\sqrt{d}}
            }
        \,
        \delta_{
            (V\mathbf{X})_n
        }
.
\]
A transformer network operates by iteratively applying deep ReLU MLPs to the rows of its input matrix and then applying the attention mechanisms matrix-wise.  Thus, a \textit{transformer with multi-head attention} is a map $\mathcal{T}:\mathbb{R}^{N(d+1)}\to \mathbb{R}^{M\times D}$ sending any $\mathbf{X}\eqdef (x_n,w_n)_{n=1}^N \in \mathbb{R}^{N(d+1)}$ to $\mathcal{T}(\mathbf{X})\in \mathbb{R}^{M\times D}\eqdef \mathbf{X}^{(J)}$, where $\mathbf{X}^{(J)}$ is defined recursively by
\begin{equation}
\label{eq:representation_transformer}
    \begin{aligned}
    \mathbf{X}^{(j+1)}
        \eqdef
    \underbrace{
        \bigoplus_{h=1}^{H_j}\,
            \operatorname{Attn}(\mathbf{Z}^{(j+1)}
            |\mathcal{Q}_h,\mathbf{K}_h,\mathcal{V}_h)
    }_{\text{Multi-head Attention}}
    \mbox{ and }
    \mathbf{Z}^{(j+1)} 
    \eqdef 
    \underbrace{
    \sigma_{\bar{\alpha}^{(j)}}\bullet(
        \mathbf{X}^{(j)}
            +
        b^{(j)})
    }_{\text{Activation and Bias}}
    %
    \end{aligned}
\end{equation}
where $\mathbf{X}^{(0)} \eqdef \mathbf{X}$, $j$ runs from $0$ to $J-1$ and, for each $j=1,\dots,J$, $b^{(j)}$ is a $d_{j+1}\times d_j$ \textit{bias matrix}, $d_0=d$, and $d_J=D$.  
Here, $J$ is called the depth of the transformer, $\max_{j=1,\dots,J-1}\,H_j$ is the maximum number of attention heads, and $\max_{j=1,\dots,J-1}d_j$ is called its width.

Our transformers mirror approximation guarantees for the MLP model, which do not include regularization layers such as skip connections or normalization. We are using the standard multi-head attention mechanism and we do \textit{not} considered in the PIC guarantees of~\cite{furuya2024transformers}, \cite{petrov2024universal}, nor a single \textit{linearized surrogate} head attention of real-world multi-head attention mechanisms, as is often studied in the statistical literature~\cite{zhang2024trained,lu2024asymptotic,kim2024transformers}, or a single attention head as in~\cite{kratsios2023approximation,kim2024transformers}.  

\section{Main Result}
\label{s:Main}

\begin{setting}
\label{Our_Setting}
Fix context lengths $N,M\in \mathbb{N}_+$, dimensions $d,D\in \mathbb{N}_+$, and subsets $\mathcal{X}\subset\mathbb{R}^d$ and $\mathcal{Y}\subset\mathbb{R}^D$ with $\mathcal{Y}$ closed.  
Let $\omega$ be a modulus of continuity and $\mathcal{K}\subseteq \mathcal{P}_{C,N}(\mathcal{X}) \times \mathcal{X}$ be a compact subset of PICs.  
Let $\mathbb{P}$ be an Ahlfors $q$-regular probability measure on $\mathcal{K}$
.
\end{setting}



\begin{theorem}[MLPs are Universal Approximators for PICs]
\label{thrm:Main__SimpleVersion}
In the Setting~\ref{Our_Setting}, for any 
 uniformly continuous $f:\mathcal{K}\to \mathcal{P}_M(\mathcal{Y})$ with an increasing modulus of continuity $\omega$.  
For each $\omega$-uniformly continuous contextual mapping $f:\mathcal{P}_{C,N}(\mathcal{X}) \times \mathcal{X} \to \mathcal{P}_M(\mathcal{Y})$
and approximation error $0<\delta\le \operatorname{diam}(\mathcal{K})$ and every confidence level $0<\delta_{\ast}<\delta$: there is a ReLU MLP $\hat{f}:\mathbb{R}^{N\times (d+1)} \times \mathbb{R}^{d}  \to \mathbb{R}^{M \times D}$ satisfying:
\begin{enumerate}
    \item[(i)] \textbf{High-Probability Uniform Estimate:} 
    We have the typical uniform estimate
    \[
            \mathbb{P}\big(
                \mathcal{W}_1\big(
                            f(\mu,x)
                            ,
                            \hat{f}(\mu,x)
                        \big)
                \le \omega(\delta)
            \big)
        \gtrsim
            1-(\delta^q-\delta_{\ast}^q).
    \]
    \item[(ii)] \textbf{Tail-
    Estimates:} On the ``bad'' set $\mathcal{R}_{\delta} \eqdef \{(\mu,x)\in \mathcal{K}:\, \mathcal{W}_1\big(
                            f(\mu,x)
                            ,
                            \hat{f}(\mu)
                        \big)> \omega(\delta)\}$: 
    we have the tail-
    estimate
    \[
            \mathbb{E}_{\mathbb{P}}\big[
                \mathcal{W}_1(f(\mu,x),\hat{f}(\mu,x))
            \big|
                \mu \in \mathcal{R}^{\star}_{\delta}
            \big]
    \lesssim 
    \omega(\delta)
    +
    (\delta^q-\delta_{\ast}^q)
    .
    \]
\end{enumerate}
If $\mathcal{K}$ is $q$-dimensional and $d$ is large enough then, the depth and width of $\hat{f}$ are $\mathcal{O}\big(
\frac{
    dN^{2C}
}{
    \omega^{-1}(\varepsilon)^{2q}
}
\big)$ and $\mathcal{O}\big(
\frac{
    dN^{2C}
}
{
    \omega^{-1}(\varepsilon)^q
}
\big)$, respectively.
\end{theorem}
Setting $M=1$ and taking $\mathcal{Y}=\mathbb{R}$, then we see that every probability measure in $\mathcal{P}_M(\mathcal{Y})$ is of the form $\delta_y$ for some $y\in \mathbb{R}$.  Thus, $\mathcal{P}_1(\mathbb{R})$ is in bijection (actually homeomorphic) with $\mathbb{R}$.  Thus, Theorem~\ref{thrm:Main__SimpleVersion} implies the simple formulation in~\eqref{thrm:Main__SimpleVersion} which is a quantitative version of that considered in~\cite{furuya2024transformers} but for the classical ReLU MLP model in place of the transformer network with single attention-head per block.

We conclude our analysis by showing that our main result for MLPs directly implies the same conclusion for transformers with multiple attention heads per transformer block.  Thus, we obtain a quantitative version of the main theorem in~\cite{furuya2024transformers} as a direct consequence of our main result.
The following key result allows us to represent MLPs as transformers with sparse MLPs.
\begin{corollary}[{Multihead Transformers version of Theorem~\ref{thrm:Main__SimpleVersion}}]
\label{cor:Main_TransformerVersion__SimpleVersion}
In the setting of Theorem~\ref{thrm:Main__SimpleVersion} and suppose that $\mathcal{K}$ is $q$-dimensional.  For $d$ large enough, there is a transformer $\hat{\mathcal{T}}:\mathbb{R}^{N\times (d+1)} \times \mathbb{R}^{d}  \to \mathbb{R}^{M \times D}$, taking values in $\operatorname{Mat}_C^{N,D}$, and satisfying:
\[
            \mathbb{P}\big(
                \mathcal{W}_1\big(
                            f(\mu,x)
                            ,
                            \hat{\mathcal{T}}(\mu,x)
                        \big)
                \le \omega(\delta)
            \big)
        \gtrsim
            1-(\delta^q-\delta_{\ast}^q).
\]
Moreover, the depth and width of $\hat{f}$ are $\mathcal{O}\big(
\frac{
    dN^{2C}
}{
    \omega^{-1}(\varepsilon)^{2q}
}
\big)$ and $\mathcal{O}\big(
\frac{
    dN^{2C}
}
{
    \omega^{-1}(\varepsilon)^q
}
\big)$, respectively, and $\hat{\mathcal{T}}$ has exactly $N$ attention heads per transformer block.
\end{corollary}

\section{Unpacking The Construction}
\label{s:ProofSketch}
We now explain our main result, namely Theorem~\ref{thrm:Main__SimpleVersion}.

\subsection{{Step 1 - Regular Decomposition of The (PI) Context Space $\mathcal{K}$}}
\label{s:ProofSketch__ss:Decomposition}

Fix a compact set $\mathcal{K}\subseteq \mathcal{P}_{C,N}(\mathcal{X})\times\mathcal{X}$.   
Figure~\ref{fig:TrifflingRegions} showcases the near piecewise constant partition of unity, which we will use to construct our approximating MLPs and multi-head transformers.  
The idea is to subdivide $\mathcal{K}$ into a $K\in \mathbb{N}_+$ parts, and to construct MLPs which implement piecewise constant function with a value on each piece.  This construction is a non-Euclidean version of the \textit{trifling regions approach} for classical ReLU MLP approximation of functions on the cube $[0,1]^d$ considered in~\cite{ZuoweiHaizhaoZhang_2022_JMPA}; there, the authors decomposed $[0,1]^d$ in disjoint (up to their boundary) sub-cubes which can be understood as \textit{Balls} with respect to the $\ell^{\infty}$ metric on $[0,1]^d$.  That classical case was simple since $\ell^{\infty}$ balls (sub-cubes) can be chosen to be disjoint (up to their boundaries), meaning that the Voronoi cells are (up to their boundary) equal to $\ell^{\infty}$ balls; e.g.\ 
\[
        [0,1]
    =
            \underbrace{\{x\in [0,1]:\, |x-1/4|\le 1/4\}}_{\ell^{\infty}\text{-Ball about } 1/4} 
        \bigcup 
            \underbrace{\{x\in [0,1]:\, |x-3/4|\le 1/4\}}_{\ell^{\infty}\text{-Ball about } 3/4}
.
\]
In contrast, in general, metric spaces such as $\mathcal{K}$ one \textit{cannot} ensure that metric balls are disjoint. 
Unlike sub-cubes of $[0,1]^d$ which align perfectly, these metric balls need not do so in $\mathcal{K}$.  
Thus, we need to manually create disjoint regions by \textit{iteratively} deleting overlapping regions between balls.   
The result, illustrated in Figure~\ref{fig:TrifflingRegions} is constructed as follows.

Let $0<\delta\le \operatorname{diam}(\mathcal{K})$, and $\delta$-packing $\{(\mu_k,x_k)\}_{k=1}^K$ of $\mathcal{K}$, and  $0\le \delta_{\ast}\le \delta$; for each $k=1,\dots,K$ recursively define the \textit{retracted Voronoi cells} $\{C_k^{\delta_{\star}}\}_{k=1}^K$ by
\begin{equation}
\label{eq:RetractedVoronoiCells}
        C_k^{\delta_{\ast}}
    \eqdef 
            B((\mu^k, x^k),\delta_{\ast})
        \setminus
            \bigcup_{j<k}\,
                B((\mu^j, x^j),\delta)
.
\end{equation}

\begin{figure}[H]
\vspace{-1em}
    \centering
    \includegraphics[width=0.25\textwidth]{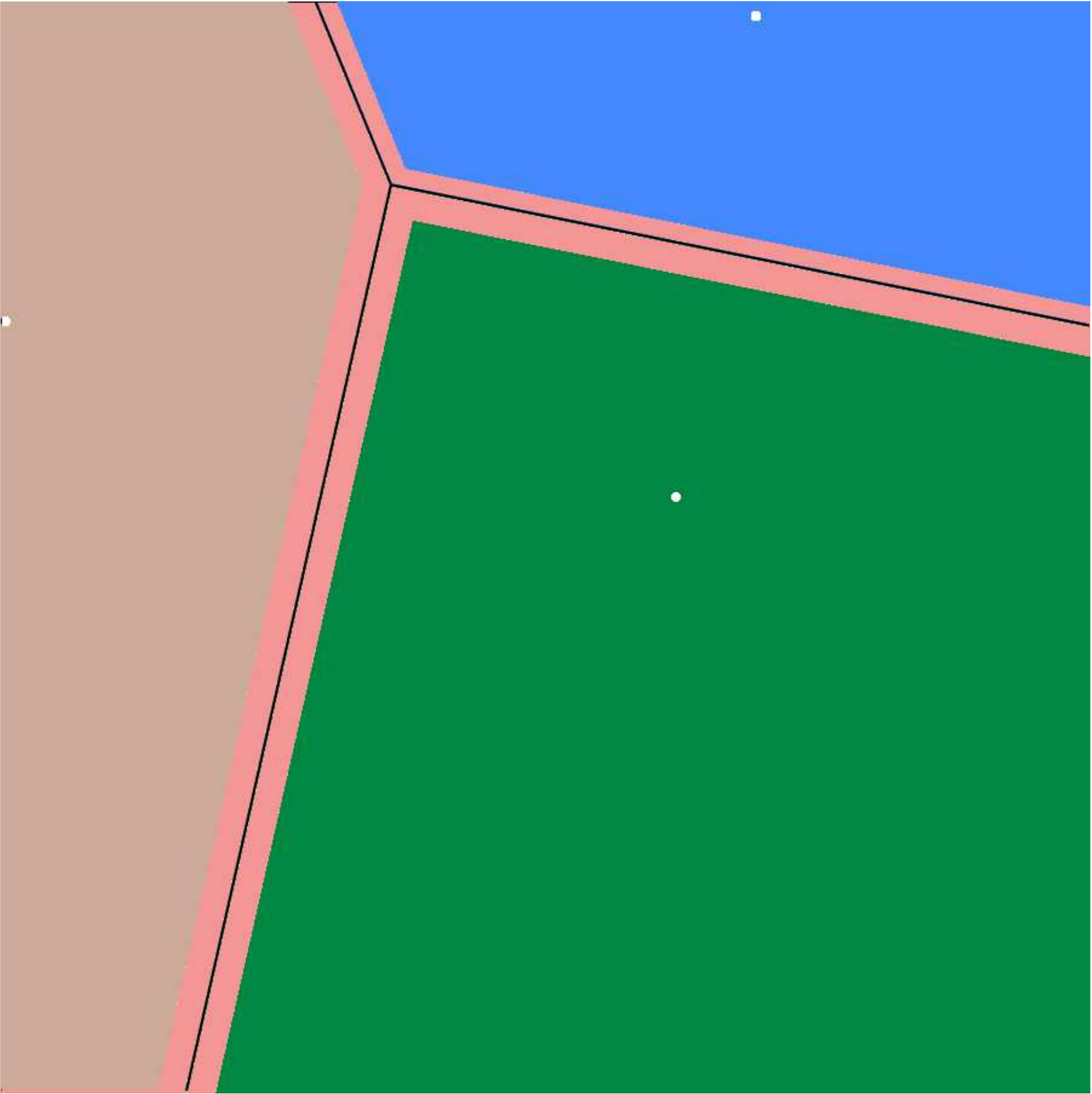}
    \caption{\textbf{Our Regular Decomposition of the PIC Space $\mathcal{K}:$} 
    The \textit{Retracted Voronoi cells} (non-redish coloured regions) $C^{\delta_{\star}}_1,\dots,C^{\delta_{\star}}_K$, for $\delta_{\star}>0$, whose union makes up the ``large'' \textit{approximation region} $\mathcal{K}\setminus \mathcal{K}^{\delta_{\star}}$.  The reddish region symbolizes our ``small'' \textit{trifling region} whereon a uniform approximation may fail.}
    \label{fig:TrifflingRegions}
\end{figure}

Our approximators will uniformly approximate the given target PIC-function $f:\mathcal{K}\to \mathcal{P}_M(\mathcal{Y})$ on the \textit{approximation region} $\mathcal{K}^{\delta_{\star}}$ covered by the retracted Voronoi cells $\{C_k^{\delta_{\star}}\}_{k=1}^K$.  Our approximator will be continuous and piecewise constant, one (possibly distinct) value on each retracted Voronoi cell $C_k^{\delta_{\star}}$, and that value is given by the target PIC-function's value at that cell's \textit{landmark point} $x^k$.  The \textit{approximation region} and \textit{trifling region} are 
\begin{equation}
\label{eq:Approx_and_trifling_Regions}
\underbrace{
       \mathcal{K}^{\delta_{\ast}}
    \eqdef
        \bigcup_{k\in [K]}\,
            [
                B((\mu^k, x^k),\delta)
                \setminus
                B((\mu^k, x^k),\delta_{\ast})
            ]
}_{\text{Trifling Region}}
\mbox{ and }
\underbrace{
       \mathcal{K}\setminus \mathcal{K}^{\delta_{\ast}}
.
}_{\text{Approximation Region}}
\end{equation}

We can ensure that the trifling region is small.  For class $L^p$-type approximation guarantees, with $1\le p<\infty$, the ``smallness'' of the trifling region in~\cite{ZuoweiHaizhaoZhang_2022_JMPA} guaranteed by it having a small Lebesgue measure.  Here, in our non-Euclidean setting, the smallness of our trifling region is guaranteed by having a small measure for any fixed Ahlfors regular measure on $\mathcal{K}$.

\begin{lemma}[The Trifling Region is Small]
\label{lem:VornoiLusin}
Let $N\in \mathbb{N}_+$, $q>0$, $\mathcal{K}\subseteq \mathcal{P}_{C,N}(\mathcal{X}) \times\mathcal{X}$ be a totally bounded subset of $\mathcal{P}_{C,N}(\mathcal{X}) \times\mathcal{X}$, and $\mathbb{P}$ be an Ahlfors $q$-regular measure on $\mathcal{K}$
For any $0<\delta\le \operatorname{diam}(\mathcal{K})$, any $\delta$-packing $\{(\mu^k,x^k)\}_{k=1}^K$ of $\mathcal{K}$, and any $0<\delta_{\ast}<\delta$ the 
\textit{approximation region} $\mathcal{K}^{\delta_{\ast}}$ is ``large'' in the sense that
\[
    \mathbb{P}\big(\mathcal{K}\setminus \mathcal{K}^{\delta_{\ast}}\big)\ge 1 - C\, K[\delta^q-\delta_{\ast}^q]
.
\]
\end{lemma}
See Appendix~\ref{sec:proof-lem:VornoiLusin} for the proof. 
To ensure the continuity of our approximator, which is piecewise constant on the approximation region $\mathcal{K}$, we note that the retracted Voronoi cells are disjoint, and each pair of distinct cells is separated by a distance of at least $\delta^{\star}$.
\begin{lemma}[{$\delta_{\ast}$-Separated Almost Partition of $\mathcal{K}$}]
\label{lem:POU_Lemma__gentrifling}
Consider the setting of Lemma~\ref{lem:VornoiLusin}.  
The retracted Voronoi cells $\{C_k^{\delta_{\ast}}\}_{k=1}^K$ satisfies:
\begin{enumerate}
    \item[(i)] \textbf{$\delta_{\ast}$-Separated:} $
        \min_{k,k'\in [K],\,k\neq k'}\,
            \mathcal{W}_1(C_k^{\delta_{\ast}},C_{k'}^{\delta_{\ast}})
        \ge 
            \delta-\delta_{\star}
        ,
    $
    \item[(ii)] \textbf{Almost Cover:} $\{C_k^{\delta_{\ast}}\}_{k\in [K]}$ covers $\mathcal{K}\setminus \mathcal{K}^{\delta_{\ast}}$,
\end{enumerate}
where for two non-empty subsets $A$ and $B$ of $\mathcal{P}_N(\mathcal{X})$ we define 
$
            \mathcal{W}_1(A,B)
        \eqdef 
        \underset{\nu\in A}{\sup}\, \underset{\nu'\in B}{\inf}
        \mathcal{W}_1(\nu,\nu')
$.
\end{lemma}

\subsection{{Step 2 - Optimal Piecewise-Constant Approximator on $\mathcal{K}$}}
\label{s:ProofSketch__ss:Optimal_Approximator}

Using our partitioning lemmata, from Step $1$, we construct optimal approximators; where optimality is in the sense of metric entropy of the domain.  The first step, is to construct an optimal piecewise constant approximator.
We then shrink the distance between each pair of parts using the $\delta^{\star}$ parameter in Lemma~\ref{lem:POU_Lemma__gentrifling}.  Doing so gives us just enough flexibility to ensure to perturb the piecewise constant functions into ``piecewise linear'' approximators which coincide with the original piecewise constant approximator on $\mathcal{K}\setminus\mathcal{K}^{\delta_{\ast}}$.  Taking $\delta_{\ast}$ small enough gives the desired approximation (but not yet implemented by our transformer).
\begin{lemma}[Optimal Piecewise Constant Approximator]
\label{lem:PWC_Approximation}
Let $N\in \mathbb{N}_+$, $q>0$, $\mathcal{K}\subseteq \mathcal{P}_{C,N}(\mathcal{X})\times\mathcal{X}$ be a totally bounded subset of $\mathcal{P}_{C,N}(\mathcal{X})\times\mathcal{X}$, and $f:\mathcal{K}\to \mathcal{P}_M(\mathcal{Y})$ is uniformly continuous with an increasing modulus of continuity $\omega$.
Let $0<\delta\le \operatorname{diam}(\mathcal{K})$ and every $\delta$-packing $\{(\mu^k, x^k)\}_{k=1}^K$ of $\mathcal{K}$ and $\{C_k^{\delta}\}_{k\in [K]}$ be 
the partition of $\mathcal{K}$ of 
in Lemma~\ref{lem:POU_Lemma__gentrifling}.   
There are $\nu_1,\dots,\nu_K \in f(\mathcal{K})$ 
and a piecewise constant function on $\mathcal{P}(\mathcal{X}) \times \mathcal{X}$
\begin{equation}
\label{eq:elementary_PWC_Approximator}
        f_{\delta}(\mu,x)
    \eqdef 
        \sum_{k=1}^K\, 
            \nu_k I_{(\mu,x) \in C_k^{\delta}}
\end{equation}
is a well-defined map from $\mathcal{K}$ to $\mathcal{P}_M(\mathcal{Y})$ satisfying $
    \sup_{(\mu,x)\, \in \mathcal{K}}\,
        \mathcal{W}_1\big(
            f(\mu,x)
            ,
            f_{\delta}(\mu,x)
        \big)
    \le 
        \omega(\delta)
$.
\end{lemma}

\subsection{{Step 3 - Implementing the Optimal Approximator $\mathcal{K}\setminus \mathcal{K}^{\delta_{\ast}}$}}
\label{s:ProofSketch__ss:Implementation}
Lemma~\ref{lem:implementation_W1} shows that ReLU MLPs can \textit{exactly} implement the $1$-Wasserstein distance for PIC in $\mathcal{P}_N(\mathcal{X})$ up to the identification in~\eqref{eq:identification}.  Thus allows us to exactly implement a piecewise constant partition of unity on $\mathcal{K}\setminus\mathcal{K}^{\delta_{\ast}}$.  

\begin{lemma}[{Piecewise Constant Partition of Unity on $\mathcal{K}\setminus\mathcal{K}^{\delta_{\ast}}$}]
\label{lem:POU__Approx-context}
In the above setting, for each $k\in [K]$, there is an MLP $\tilde{\Phi}_k:\mathbb{R}^{N\times (d+1)} \times \mathbb{R}^{d} \to [0,1]$ satisfying the following for each $(\mu, x)\in \mathcal{K}\setminus\mathcal{K}^{\delta_{\ast}}$
\begin{align}
\label{eq:near_representatation_Klarge_case-context}
        \tilde{\Phi}_k(\mu,x)
     =
        I_{(\mu, x) \in C_k^{\delta_{\ast}}}
.
\end{align}
Moreover, $\tilde{\Phi}_k$ 
is 
$\operatorname{depth}(\tilde{\Phi}_j) \in \mathcal{O}(d + N^{N+2C-3})$ and $\operatorname{width}(\tilde{\Phi}_j)\in \mathcal{O}(Nd + N! \, N^{2C-2})$.
\end{lemma}
See Appendix~\ref{sec:proof-lem:POU__Approx-context} for the proof. The idea is to use Lemma~\ref{lem:implementation_W1}, showing that ReLU MLPs can \textit{exactly} implement the $1$-Wasserstein distance for PIC in $\mathcal{P}_N(\mathcal{X})$ up to the identification in~\eqref{eq:identification}, to implement a piecewise constant partition of unity $\{I_{C_k^{\delta_{\star}}}\}_{k=1}^K$ using MLPs outside the trifling region.

\subsection{{Step 4 - Proof of Theorem~\ref{thrm:Main__SimpleVersion}}}
\label{s:ProofSketch__ss:Completion}
Finally, it is enough to show the following lemma to prove the main result.
\begin{lemma}
\label{lem:UAT__NoTransformerYet}
In the above setting, there exists a ReLU MLP $\hat{f}_{\delta} : \mathbb{R}^{N \times (d+1)} \times \mathbb{R}^d \to \mathbb{R}^{N \times D}$ of depth $\mathcal{O}(K^2(d+N^{N+2C-3}))$ and width $\mathcal{O}(NdK + N! N^{2C-2})$ satisfying:
\begin{enumerate}
    \item[(i)] \textbf{Uniform Estimation:} 
    $
        \sup_{(\mu,x) \in \mathcal{K} \setminus \mathcal{K}^{\delta_{\ast}}}
        \,
            \mathcal{W}_1\big(
                        f(\mu)
                        ,
                        \hat{f}_{\delta}(\mu)
                    \big)
            \le \omega(\delta)
    $
    \item[(ii)] \textbf{Probability of Estimated Satisfaction:} 
   $
            \mathbb{P}(
                 \mathcal{K}\setminus \mathcal{K}^{\delta_{\ast}}
            )
        \gtrsim
            1 - 
            (\delta^{q}-\delta_{\ast}^q)
    $
    \item[(iii)] \textbf{Tail Moment Estimates:} For every $1\le p<\infty$ we have the tail-moment estimate
    \[\resizebox{1\linewidth}{!}{$
                \biggl(
                    \int_{(\mu,x) \in \mathcal{K}^{\delta_{\ast}}}\,
                        \mathcal{W}_1(f(\mu,x),\hat{f}_{\delta}(\mu,x))^p
                        \,
                        \mathbb{P}(d\mu dx)
                \biggr)^{1/p}
            \lesssim 
                \left(
            \omega(\delta)
            +
            \max_{k\in [K]}\,
            \|f(\mu^k, x^k)\|_{KR}
            \,
        \right) (\delta^q-\delta_{\ast}^q)^{1/p}
    $}\]
\end{enumerate}
\end{lemma}
See Appendix~\ref{sec:proof-lem:UAT__NoTransformerYet} for the proof. 
The simple version of Lemma~\ref{lem:UAT__NoTransformerYet} (iii) in Theorem~\ref{thrm:Main__SimpleVersion} (ii), is obtained by setting $p=1$ and noting that 
$\max_{k\in [K]}\,
    \|f(\mu^k, x^k)\|_{KR}
\le \sup_{(\mu,x)\in \mathcal{K}}<\infty$ by compactness of $\mathcal{K}$ and continuity of $f$.
Finally, the ``simple'' quantitative expression in Theorem~\ref{thrm:Main__SimpleVersion} is obtained by considering the high-dimensional setting where $N^N\le d$.  Since $N!\le N^N$, then both the depth and width estimates in Lemma~\ref{lem:UAT__NoTransformerYet} are as in Theorem~\ref{thrm:Main__SimpleVersion}.

It remains bound to $K$.  If $\mathcal{K}$ is additionally assumed to have a finite doubling dimension $0\le q<\infty$, then there is a $C>0$ such that every ball of radius $2r>0$ in $\mathcal{K}$ can be covered by at-most $C(2r/r)^q$ balls of radius $r$.  By~\cite[Lemma 7.1]{acciaio2024designing}, the minimal number $K$ of balls of radius $\omega^{-1}(\delta)$ needed to cover $\mathcal{K}$ is at-most 
$
    \log(K)
\le
    q\lceil \log\big(\frac{\operatorname{diam}(\mathcal{K})}{\omega^{-1}(\delta)}\big)\rceil
    \,
    \log(C2)
$.  Thus, $K\in \mathcal{O}\big(1/ \omega^{-1}(\varepsilon)^q\big)$.  When $f$ is $\alpha$-H\"{o}lder then $K\in \mathcal{O}(\varepsilon^{-q/\alpha})$; analogously to the well-known optimal approximation rates for ReLU MLPs when approximating classical functions (not in-context) on the $q$-dimensional Euclidean space~\cite{yarotsky18a_verydeep_COLT}.

\paragraph{From MLPs to Transformers}
Similar to the strategy of~\cite{petersen2020equivalence,singh2023expressivity}, we 
shows that every ReLU MLP can be converted into a transformer in a canonical fashion.  This allows us to deduce a quantitative version of the in-context universality results of~\cite{furuya2024transformers} for the transformer model.
\begin{proposition}[Transformerification]
\label{prop:transformerification__SparseVersion}
Let $f:\mathbb{R}^{N\times d}\to \mathbb{R}^D$ be an MLP with depth $J$, width $W$, $K$ non-zero (trainable) parameters, and with trainable activation function $\sigma$ as in~\eqref{eq:activation}.
Then, $f$ can be implemented as a transformer MLP with the same depth and width as $f$, exactly $N$ attention heads at each layer, and at most $2K$ non-zero parameters.
\end{proposition}

\section{Conclusion}
In conclusion, our results show that the classical MLP shares the in-context universality of the transformer for PICs. Thus, the transformer’s success must be attributed to other factors, such as its inductive bias or ability to leverage context.

\bibliography{Bookkeeping/Bibliography/refs,Bookkeeping/Bibliography/InContext} 

\newpage

\appendix

\acks{
A.\ Kratsios acknowledges financial support from an NSERC Discovery Grant No.\ RGPIN-2023-04482 and No.\ DGECR-2023-00230.  A.\ Kratsios also acknowledges that resources used in preparing this research were provided, in part, by the Province of Ontario, the Government of Canada through CIFAR, and companies sponsoring the Vector Institute\footnote{\href{https://vectorinstitute.ai/partnerships/current-partners/}{https://vectorinstitute.ai/partnerships/current-partners/}}.

T.\ Furuya was supported by JSPS KAKENHI Grant Number JP24K16949, JST CREST JPMJCR24Q5, JST ASPIRE
JPMJAP2329, and Grant for Basic Science Research Projects from The Sumitomo Foundation. 
Significant progress on this paper was made during the Visit of T.\ Furuya to McMaster in December $2024$.

The authors would also like to thank \href{https://scholar.google.ca/citations?user=bns0iwUAAAAJ&hl=en}{Behnoosh Zamanlooy} for her very useful feedback in the final stages of the manuscript, and both \href{https://www.gpeyre.com/}{Gabriel Peyr\'e} and \href{https://maartendehoop.rice.edu/}{Maarten de Hoop} for their helpful feedback throughout the early stages of this project.
}

\section{Details on Trainable Activation Function}
\label{a:Activ}
The following are some simple configurations, i.e.\ parameter choices of the trainable activation function in~\eqref{eq:activation}, which recover standard activation functions/neural network layers.   We first remark that this trainable activation function can implement the well-studied rectified linear unit (ReLU).
\begin{example}[ReLU]
\label{ex:ReLU}
If $\theta=(1,1,0)$ then $\sigma_{(1,1,0)}
=\max\{0,x\}$.
\end{example}
The trainable activation function in~\eqref{eq:activation} allows one to implement skip connections.  Skip connections are standard in most residual neural networks, transformer networks, and even recent graph neural network models~\cite{borde2024scalable}, due to their regularizing effect on the loss landscape of the deep learning models~\cite{riedi2023singular}.
\begin{example}[Gated/Skip Neuron]
\label{ex:ID}
If $\theta=(1,1,1)$ then $\sigma_{(1,1,1)}=x
$.
\end{example}
One can easily recover the rectified quadratic/power unit (ReQU); which has garnered significant recent attention due to its approximation potential; see e.g.~\cite{belomestny2023simultaneous,furuya2024simultaneously,shen2024nonparametric}.
\begin{example}[ReQU]
\label{ex:ReQU}
If $\theta=(1,2,0)$ then $\sigma_{(1,2,0)}
=\max\{0,x\}^2$.
\end{example}
We do not allow for $p\le 0$, thus our neural network using~\eqref{eq:activation} are always continuous; however, if one set $\theta=(1,0,0)$ then we would obtain the classical heavy-side function.
\begin{example}[Binary/Threshold]
\label{ex:hard_step}
If $\theta=(1,0,0)$ then $\sigma_{(1,0,0)}=
I_{(0,\infty)}$.
\end{example}

\section{Conversion: MLP to Multi-Head Transformers}
\label{b:transformers}

The proof of Proposition~\ref{prop:transformerification__SparseVersion} relies on the following lemma.
\begin{lemma}[Transformers Implement Fully-Connected Weights]
\label{lem:fullyconnected_MatMul_viaMHAttention}
Let $N,\tilde{d}\in \mathbb{N}_+$ and $B\in \mathbb{R}^{N\times \tilde{d}}$.  Then, there exists a multi-head attention mechanism $\bigoplus_{h\in H}\,\operatorname{Attention}(\cdot|Q_h,K_h,V_h,\lambda)$ with exactly $N$ heads and at-most $2N+N\tilde{d}$ non-zero parameters such that
\[
    \bigoplus_{h\in H}\,
        \operatorname{Attention}(\mathbf{X}|Q_h,K_h,V_h,\lambda)
    =
        B\,\mathbf{X}
.
\]
\end{lemma}
\begin{proof}
Fix $\tilde{d},h\in \mathbb{N}_+$ with $\tilde{d}\le h$.
Let $\mathbf{E}^{h:\tilde{d}}$ be the elementary $\tilde{d}\times \tilde{d}$ with $\mathbf{E}^{h:\tilde{d}}_{i,j}\eqdef I_{i=j=h}$ for each $i,j=1,\dots,\tilde{d}$.  
Let $H\eqdef N$ and for each $h \in [H]$ define 
\[
    K_h\eqdef \mathbf{0}_{\tilde{d}}
    ,\, 
    Q_h\eqdef  \mathbf{0}_{\tilde{d}}
    \mbox{, and }
    V_h\eqdef  (N\, \mathbf{E}^{h:\tilde{d}} B)
.
\]
Then, for each $\lambda>0$, each $h\in [H]$, and each $\mathbf{X}\in \mathbb{R}^{N\times \tilde{d}}$ the induced attention mechanism computes
\allowdisplaybreaks
\begin{align*}
    \operatorname{Attention}(\mathbf{X}|Q_h,K_h,V_h,\lambda)
& \eqdef 
    \sum_{m\in [N]}
    \,
        \frac{
            e^{\lambda \langle K_h \mathbf{X}_n,Q_h \mathbf{X}_m\rangle }
        }{
        \sum_{k\in [N]}\,
            e^{\lambda \langle K_h \mathbf{X}_n,Q_h \mathbf{X}_k\rangle }
        }
        \,
        (V_h \mathbf{X})_n
\\
& =
        \sum_{m\in [N]}
        \,
        \frac{
            e^{ 0 }
        }{
        \sum_{k\in [N]}\,
            e^{0}
        }
        \,
        ((N\,\mathbf{E}^{h:\tilde{d}} B) \mathbf{X})_n
\\
& =
        \sum_{m\in [N]}
        \,
        \frac1{N}
        \,
        N((B\mathbf{X})_n I_{n=h})
\\
& =
        (B\mathbf{X})_n
.
\end{align*}
Then, for each $\lambda>0$ the associated multi-head attention mechanism computes the following, for each $\mathbf{X}\in \mathbb{R}^{N\times \tilde{d}}$ 
\[
    \bigoplus_{h\in H}\,
        \operatorname{Attention}(\mathbf{X}|Q_h,K_h,V_h,\lambda)
    =
        \bigoplus_{h\in N}\,
            (B\mathbf{X})_n
    =
        B\,\mathbf{X}
.
\]
This concludes the proof.
\end{proof}
\begin{proof}[{Proof of Proposition~\ref{prop:transformerification__SparseVersion}}]
Let $\hat{f}:\mathbb{R}^{N\times d}\to\mathbb{R}^D$ admit the iterative representation
\begin{equation}
\label{eq:representation_MLP_}
    \begin{aligned}
    f(x)
    & \eqdef 
    A^{(J)} X^{(J)}
    \\
    \mathbf{X}^{(j+1)} 
    &
    \eqdef 
    A^{(j)} \sigma_{\bar{\alpha}^{(j)}}\bullet(
        \mathbf{X}^{(j)}
            +
        b^{(j)})
    \qquad 
            \mbox{for }  
        j=0,\dots,J-1,    
    \\
    \mathbf{X}^{(0)} &\eqdef \mathbf{X}
.
    \end{aligned}
\end{equation}
Applying Lemma~\ref{lem:fullyconnected_MatMul_viaMHAttention} once for each $j\in [J]$, to each matrix $A^{(j)}$ implies that for each such $j\in [J]$ there is a multi-head attention mechanism such that (for any $\lambda>0$)
\begin{equation}
\label{eq:attention_implements_fullyconnectedweightmatrix}
        \bigoplus_{h\in N}\,
        \operatorname{Attention}(\mathbf{Z}|Q_h^j,K_h^j,V_h^j,\lambda)
    =
        A^{(j)}\,\mathbf{Z},
\end{equation}
for all $\mathbf{Z}\in \mathbb{R}^{d_{j+1}\times d_j}$.  Therefore,~\eqref{eq:attention_implements_fullyconnectedweightmatrix} implies that $f$ admits the representation
\allowdisplaybreaks
\begin{align*}
    f(x)
    & \eqdef 
    \bigoplus_{h\in N}\,
        \operatorname{Attention}\big(
            \mathbf{X}^{(J)}
        |Q_h^j,K_h^j,V_h^j,\lambda\big)
    \\
    \mathbf{X}^{(j+1)} 
    &
    \eqdef 
    \bigoplus_{h\in N}\,
        \operatorname{Attention}\big(
            \sigma_{\bar{\alpha}^{(j)}}\bullet(
                \mathbf{X}^{(j)}
                    +
                b^{(j)})
        |Q_h^j,K_h^j,V_h^j,\lambda\big)
    \qquad 
            \mbox{for }  
        j=0,\dots,J-1,    
    \\
    \mathbf{X}^{(0)} &\eqdef \mathbf{X}
.
\end{align*}
This concludes our proof.
\end{proof}


\section{Proof of Main Result}
\label{s:Proofs}







\subsection{Proof of Lemma~\ref{lem:VornoiLusin}}
\label{sec:proof-lem:VornoiLusin}
\begin{proof}
By the elementary inequality between covering and packing numbers, e.g.\ on~\cite[page 98]{vanderVaartWellner_2023_WkConvEmpProcessesLoveThisBook}, we know that $\{B((\mu^k, x^k),\delta)\}_{k\in [K]}$ covers $\mathcal{K}$; i.e.~$\mathcal{K}
    \subseteq 
        \bigcup_{k\in [K]}\, B((\mu^k, x^k),\delta)$.
The sub-additivity of the measure $\mathbb{P}$ then implies (i) since
\allowdisplaybreaks
\begin{align*}
\numberthis
\label{eq:Begin_AhlforsInequalities}
        \mathbb{P}(\mathcal{K}^{\delta_{\ast}})
    &
    \le 
        \sum_{k\in [K]}
        \,
        \mathbb{P}\big(
            B((\mu^k, x^k),\delta)
                \setminus
                B((\mu^k, x^k),\delta_{\ast})
        \big)
\\
    &=
    \sum_{k\in [K]}
    \,
    \mathbb{P}\big(
        B((\mu^k, x^k),\delta)
    \big)
    -
    \mathbb{P}\big(
            B((\mu^k, x^k),\delta_{\ast})
    \big)
\\
\numberthis
\label{eq:Ahlfors_used}
    &\le
        \sum_{k\in [K]}
        \,
        (C\delta^{q}
        -
        c\delta_{\ast})^q
\\
\numberthis
\label{eq:cC}
    &\le
        K
        \,
        C \delta^p - c \delta_*
\end{align*}
where we deduced~\eqref{eq:Ahlfors_used} using the definition of Ahlfors $q$-regularity and~\eqref{eq:cC} followed since $c\le C$.
Since $\mathbb{P}$ is a probability measure supported on $\mathcal{K}$ then the bounds in~\eqref{eq:Begin_AhlforsInequalities}-\eqref{eq:cC} imply that
\[
        \mathbb{P}\big(
            \mathcal{K}
            \setminus
            \mathcal{K}^{\delta_{\ast}}
        \big)
    =
        \mathbb{P}\big(
            \mathcal{K}
        \big)
        -
        \mathbb{P}\big(
            \mathcal{K}^{\delta_{\ast}}
        \big)
    \ge 
            1
        -
             K
            \,
            C\big(
                \delta^{q}
                -
                \delta_{\ast}^q
            \big)
.
\]
This completes our proof.
\end{proof}



\subsection{Proof of Lemma~\ref{lem:POU_Lemma__gentrifling}}
\label{sec:proof-lem:POU_Lemma__gentrifling}

\begin{proof}
For (i), observe that for each $k\in [K]$, $C_k^{\delta_{\ast}}\subseteq B((\mu^k, x^k),\delta_{\ast})$ and by definition of \\
$\mathcal{W}_1\big(
B((\mu^k, x^k),\delta_{\ast})
,
B((\mu^{k'}, x^{k'}),\delta)
\big)$ for every $k'\in [K]$ for which $k\neq k'$.  
Since any $\delta$-packing is a $\delta$-net, then $\{B((\mu^k, x^k),\delta)\}_{k\in [K]}$ covers $\mathcal{K}$; (i) then holds by definition of $\mathcal{K}^{\delta_{\ast}}$.
\end{proof}


\subsection{Proof of Lemma~\ref{lem:PWC_Approximation}}
\label{sec:proof-lem:PWC_Approximation}

\begin{proof}
For each $k\in [K]$ define 
$$
\nu_k\eqdef f(\mu^k,x^k).
$$
By Lemma~\ref{lem:POU_Lemma__gentrifling}, $\{C_k^{\delta}\}_{k\in [K]}$ is a partition of $\mathcal{K}$.  
Therefore, for every $(\mu, x) \in \mathcal{K}$ there exists exactly one $k_0 \in[K]$ such that $I_{(\mu,x)\in C_{k_0}^{\delta}}$ is non-zero.  Since the codomain of $f$ is $\mathcal{P}_M(\mathcal{Y})$ then, by definition, for each $k\in [K]$, $\nu_k\in \mathcal{P}_M(\mathcal{Y})$.  Thus, for each $(\mu, x) \in \mathcal{K}$ we have
\[
    f_{\delta}(\mu,x) 
    = 
    \nu_{k_0},
\]
whence $f_{\delta}$ also takes values in $\mathcal{P}_M(\mathcal{Y})$.
Now, we have that: for each $(\mu,x) \in \mathcal{K}$
\allowdisplaybreaks
\begin{align}
\nonumber
        \mathcal{W}_1\big(
                f(\mu,x)
            ,
                f_{\delta}(\mu,x)
        \big)
    & =
        \mathcal{W}_1\big(
            f(\mu,x)
        ,
            \nu_{k_0}
    \big)
\\
\nonumber
    & =
        \mathcal{W}_1\big(
            f(\mu,x)
        ,
            f(\mu^{k_0}, x^{k_0})
    \big)
\\
\label{eq:UC}
    & \le
        \omega\big(
            \mathcal{W}_1(
                \mu
            ,
                \mu^{k_0}
            )
            + \|x-x^{k_0}\|
        \big),
\end{align}
where~\eqref{eq:UC} held by the uniform continuity of $\omega$.  
Since $C_{k_0}^{\delta}\subseteq B((\mu^{k_0},x^{k_0}),\delta)$ and $(\mu,x) \in C_{k_{0}}^{\delta}$ then, the monotonicity of $\omega$ and~\eqref{eq:UC} imply that $
        \mathcal{W}_1\big(
                f(\mu,x)
            ,
                f_{\delta}(\mu,x)
        \big)
 \le
        \omega(\delta)
$.
\end{proof}

\subsection{Proof of Lemma~\ref{lem:POU__Approx-context}}
\label{sec:proof-lem:POU__Approx-context}

\begin{proof}
\hfill\\
\noindent \textbf{Step 1 - Base Case:}
\hfill \\
If $k=1$ then consider the indicator function $I_{\mu \in B((\mu^1,x^1),\delta_{\ast}))}$ of $B((\mu^1, x^1),\delta_{\ast}))$.
We may represent this function as follows: for any $\mu\in \mathcal{K}\setminus\mathcal{K}^{\delta_{\ast}}$ as
\begin{equation}
\label{eq:near_representation__K1_case-context}
        \phi_{\sqcap:\delta_{\ast},\delta}(\mathcal{W}_1(\mu,\mu^1) + \|x-x^1\|_2^2)
    = 
        I_{\mu \in B((\mu^1,x^1),\delta_{\ast}))}
\end{equation}
where $\phi_{\sqcap:\delta_{\ast},\delta}(x_i) $ is defined by 
\begin{equation}
\label{eq:PW_Linear_Bump}
        \phi_{\sqcap:\delta_{\ast},\delta}
        (t) 
    \eqdef 
        \begin{cases}
            1 & |t|\le \delta_{\ast} \\
            \frac{1}{\delta_{\ast}-\delta}(|t|-\delta) & \delta_{\ast}\le t\le \delta \\
            0 & |t|>\delta
.
        \end{cases}
\end{equation}
By Lemmas~\ref{lem:implementation_W1} and \ref{eq:norm_implementation__l2}, there exists a ReLU MLPs $\Phi_{W^1}:\mathbb{R}^{N\times (d+1)}\times \mathbb{R}^{N\times (d+1)}\to [0,\infty)$ and $\Phi_{\|\cdot\|_2^2}:\mathbb{R}^{d} \to [0,\infty)$ implementing $\mathcal{W}_1$ on $\mathcal{P}_{C,N}(\mathbb{R}^d)$ and $\|\cdot\|_2^2$ on $\mathbb{R}^d$, respectively.
Thus, the map $\mu \mapsto \Phi_{W^1}(\mu,\mu^1)$ and $x \mapsto \Phi_{\|\cdot\|_2^2}(x-x^1)$ are MLP mapping $\mathcal{P}_{C,N}(\mathbb{R}^d)$ to $[0,\infty)$ implementing the map $\mu\mapsto \mathcal{W}_1(\mu,\mu^1)$ and MLP mapping $\mathbb{R}^d$ to $[0,\infty)$ implementing the map $x \mapsto\|x\|_2^2$, respectively.
Therefore, Lemma~\ref{lem:bump} implies that $I_{(\mu,x)\in B((\mu^1,x^1),\delta))}$ may be written as 
\begin{align}
\label{eq:near_representation__K1_case}
        \tilde{\Phi}_1(\mu,x)
    &
    \eqdef \Phi_{\sqcap:\delta_{\ast},\delta} \circ
        (\Phi_{W^1}(\mu, \mu^1) + \Phi_{\|\cdot\|_2^2}(x-x^1))
    \\
    &
    =
        \phi_{\sqcap:\delta_{\ast},\delta}\circ (W^1(\mu, \mu^1) + \|x-x^1\|_2^2)
    =
        I_{\mu\in B((\mu^1,x^1),\delta)},
\end{align}
for all $(\mu,x) \in \mathcal{K}\setminus\mathcal{K}^{\delta_{\ast}}$.  The ReLU MLP $\tilde{\Phi}_1:\mathbb{R}^{N\times d} \times \mathbb{R}^d \to [0,1]$ has, by Lemma~\ref{lem:parallelization}
\begin{enumerate}
    \item Depth $\mathcal{O}(d + N^{N+2C-3})$
    \item Width $\mathcal{O}(Nd + N! \, N^{2C-2})$.
\end{enumerate}

\noindent \textbf{Step 2 - Recursion:}
\hfill \\
For each $k>1$ and $k\in [K]$, we recursively define a ReLU MLP $\tilde{\Phi}_k:\mathbb{R}^{N\times d}  \times \mathbb{R}^d \to [0,1]$ by 
\allowdisplaybreaks
\begin{align*}
\label{eq:near_representatation_Klarge_case}
        \tilde{\Phi}_k(\mu,x)
    & \eqdef 
        \phi_{\times}
        \circ 
        \biggl(
        \Phi_{\sqcap:\delta_{\ast},\delta}\circ
         ( \Phi_{W^1}(\mu, \mu^k) + \Phi_{\|\cdot\|_2^2}(x-x^k) ),
            \,
            \bigoplus_{j < k}
        \Phi_{\sqcap:\delta_{\ast},\delta}\circ
                \big(
                    1-\tilde{\Phi}_j(\mu,x)
                \big)
        \biggr)
        \\
        & =  
        \phi_{\times}
        \circ 
        \biggl(
            \phi_{\sqcap:\delta_{\ast},\delta}\circ ( \Phi_{W^1}(\mu, \mu^k) + \Phi_{\|\cdot\|_2^2}(x-x^k) ),
            \,
            \bigoplus_{j < k}
                \phi_{\sqcap:\delta_{\ast},\delta}
                \big(
                    1-\tilde{\Phi}_j(\mu,x)
                \big)
        \biggr),
\end{align*}
where $\phi_{\times}$ is the ReQU MLP with depth $\mathcal{O}(k)$ and width $\mathcal{O}(k)$ defined in Lemma~\ref{lem:mult} and implementing the map componentwise multiplication map $
\phi_{\times}((x,y)) = (x_iy_i)_{i=1}^k,
$ for every $x,y\in \mathbb{R}^k$.  
For each $j=1,\dots,k-1$, $\tilde{\Phi}_j$ is a ReLU MLP; thus, so is $1-\tilde{\Phi}_j(\mu)$ moreover it has the same depth and width as $\tilde{\Phi}_j$.  
Again by Lemma~\ref{lem:bump}, for $j<k$ and $j\in \mathbb{N}_+$, the map $\mathcal{P}_{C,N}(\mathbb{R}^d) \times \mathbb{R}^d \ni (\mu,x) \mapsto \phi_{\sqcap:\delta_{\ast},\delta}
                \big(
                    1-\tilde{\Phi}_j(\mu,x)
                \big) \in [0,1]$ is representable by an MLP with activation function~\eqref{eq:activation} with the same order of depth and width as $1-\tilde{\Phi}_j(\mu,x)$; and therefore as $\tilde{\Phi}_j$.

Applying Lemma~\ref{lem:parallelization}, we have that the map $(\mu,x) \mapsto \biggl(
            \phi_{\sqcap:\delta_{\ast},\delta}\circ ( \Phi_{W^1}(\mu, \mu^k) + \Phi_{\|\cdot\|_2^2}(x-x^k) ),
            \,
            \bigoplus_{j < k}
                \phi_{\sqcap:\delta_{\ast},\delta}
                \big(
                    1-\tilde{\Phi}_j(\mu,x)
                \big)
        \biggr)$ is itself a ReLU MLP of
\begin{enumerate}
    \item[(i)] \textbf{Depth}: at most $\sum_{j\in [k]}\,\operatorname{depth}(\tilde{\Phi}_j) +1$,
    \item[(ii)] \textbf{Width}: at most $kd +\max_{j\in [k]} \operatorname{width}(\tilde{\Phi}_j)$.
\end{enumerate}
Therefore, the map $\tilde{\Phi}_k$ is implementable by an MLP with activation function $\sigma = ReQU$ and 
\begin{enumerate}
    \item[(i)] \textbf{Depth}: at most $\mathcal{O}\big(k+\sum_{j\in [k]}\,\operatorname{depth}(\tilde{\Phi}_j) +1))$,
    \item[(ii)] \textbf{Width}: at most $\mathcal{O}\big(k+kd +\max_{j\in [k]} \operatorname{width}(\tilde{\Phi}_j)) $.
\end{enumerate}
Now, by Lemma~\ref{lem:bump} (ii),~\eqref{eq:near_representation__K1_case}, and the definition of the sets $C_k^{\delta_{\ast}},\dots,C_k^{\delta_{\ast}}$, we have
\begin{align*}
        \tilde{\Phi}_k(\mu)
    & =
        \phi_{\sqcap:\delta_{\ast},\delta}\circ ( \Phi_{W^1}(\mu, \mu^k) + \Phi_{\|\cdot\|_2^2}(x-x^k) )
        \times 
        \prod_{j<k}
        \,
        \phi_{\sqcap:\delta_{\ast},\delta}
        \,
        \big(
            1-\tilde{\Phi}_j(\mu)
        \big)
    \\
    & =
        I_{(\mu, x) \in B((\mu^k, x^k),\delta_{\ast})}
        \times 
        \prod_{j<k}
        \,
        \big(
            1-I_{(\mu, x) \in C_j^{\delta_{\ast}}}
        \big)
    \\
    & =
        I_{\mu \in C_k^{\delta_{\ast}}}
\end{align*}
for all $(\mu,x) \in \mathcal{K}\setminus\mathcal{K}^{\delta_{\ast}}$.  This completes our proof.
\end{proof}

\subsection{Proof of Lemma~\ref{lem:UAT__NoTransformerYet}}
\label{sec:proof-lem:UAT__NoTransformerYet}

\begin{proof}

We define $\hat{f}_{\delta} : \mathbb{R}^{N \times (d+1)} \times \mathbb{R}^d \to \mathbb{R}^{N \times D}$ by 
\[
\hat{f}_{\delta}(\mu,x) 
    \eqdef 
        \sum_{k\in [K]}\, \nu_k \tilde{\Phi}_k(\mu,x)
\]
where MLPs $\tilde{\Phi}_k:\mathbb{R}^{N\times (d+1)} \times \mathbb{R}^{d} \to \mathbb{R}$ are defined in Lemma~\ref{lem:POU__Approx-context}. 
By using Lemma~\ref{lem:parallelization}, $\hat{f}_{\delta}$ is the ReLU MLP with depth $\mathcal{O}(K^2(d+N^{N+2C-3}))$ and width $\mathcal{O}(NdK + N! N^{2C-2})$.
By Lemma~\ref{lem:POU__Approx-context}, $\hat{f}_{\delta}$ coincides with $f_{\delta}$ (defined in \eqref{eq:elementary_PWC_Approximator}) on $\mathcal{K}\setminus \mathcal{K}^{\delta_{\ast}}$. Thus, by Lemma~\ref{lem:PWC_Approximation}
, we have that
\[
    \sup_{(\mu,x) \in \mathcal{K}\setminus \mathcal{K}^{\delta_{\ast}}}
    \,
        \mathcal{W}_1\big(
            f(\mu,x)
            ,
            \hat{f}_{\delta}(\mu,x)
        \big)
    \le \omega(\delta_{\ast})\le \omega(\delta)
.
\]
By Lemma~\ref{lem:VornoiLusin}, $\mathbb{P}(\mathcal{K}\setminus \mathcal{K}^{\delta_{\ast}})\ge 1 - CK[\delta^{q}-\delta_{\ast}^q]$ for some constant $C>0$. Thus, (i) and (ii) hold.
\medskip

It remains to show (iii).  For each (finite) $p\ge 1$
\allowdisplaybreaks
\begin{align}
\label{eq:uniform_estimate__v2__BEGIN}
&
    \Big(
        \int_{(\mu,x) \in \mathcal{K}^{\delta_{\ast}}}\,
            \mathcal{W}_1(f(\mu,x),\hat{f}_{\delta}(\mu,x))^p
            \,
            \mathbb{P}(d\mu dx)
    \Big)^{1/p}
\\
\le &
\,
    \Big(
        \int_{(\mu,x) \in \mathcal{K}^{\delta_{\ast}}}\,
            \mathcal{W}_1(f(\mu,x),f_{\delta}(\mu,x))^p
            \,
            \mathbb{P}(d\mu dx)
    \Big)^{1/p}
+
    \Big(
        \int_{(\mu,x) \in \mathcal{K}^{\delta_{\ast}}}\,
            \mathcal{W}_1(f_{\delta}(\mu, x),\hat{f}_{\delta}(\mu,x))^p
            \,
            \mathbb{P}(d\mu dx)
    \Big)^{1/p}
\\
\numberthis
\label{eq:uniform_estimate__v1}
\lesssim &
\,
    \omega(\delta)
    (\delta^q-\delta_{\ast}^q)^{1/p}
+
    \underbrace{
        \Big(
            \int_{(\mu,x) \in \mathcal{K}^{\delta_{\ast}}}\,
                \mathcal{W}_1(f_{\delta}(\mu,x),\hat{f}_{\delta}(\mu,x))^p
                \,
                \mathbb{P}(d\mu dx)
        \Big)^{1/p}
    }_{\term{t:tail_bound_completion}}
\end{align}
where we have used Lemma~\ref{lem:PWC_Approximation} and $\mathbb{P}(\mathcal{K}^{\delta_{\ast}})\lesssim [\delta^q-\delta_{\ast}^q]$ (Lemma~\ref{lem:VornoiLusin}).

In order to bound term~\eqref{t:tail_bound_completion}, we first note that $\{C_k^{\delta_{\ast}}\}_{k\in [K]}$ portions $\mathcal{K}\setminus \mathcal{K}^{\delta_{\ast}}$ then, we have that
\allowdisplaybreaks
\begin{align*}
        \eqref{t:tail_bound_completion}
    & =
    \Big(
        \sum_{k\in [K]}\,
        \int_{((\mu,x) \in \mathcal{K}^{\delta_{\ast}}}\,
            I_{C_k^{\delta} \setminus C_k^{\delta_{\ast}}}
            \,
            \mathcal{W}_1(f_{\delta}(\mu,x),\hat{f}_{\delta}(\mu,x))^p
            \,
            \mathbb{P}(d\mu dx)
    \Big)^{1/p}
\\
\numberthis
\label{eq:definitions_of_PWC_and_PWLinear_Approxmiators}
    & \le
    \Big(
        \sum_{k\in [K]}\, 
        \int_{(\mu,x) \in \mathcal{K}^{\delta_{\ast}}}\,
            I_{C_k^{\delta} \setminus C_k^{\delta_{\ast}}}
            \,
                \|\nu_k\|^p_{KR}
            \,
            \mathbb{P}(d\mu dx)
    \Big)^{1/p}
\\
    & \le
    \max_{k\in [K]}\,
    \|f(\mu^k,x^k)\|_{KR}\,
    \Big(
        \max_{k\in [K]}\, 
        \int_{(\mu,x) \in \mathcal{K}^{\delta_{\ast}}}\,
            I_{C_k^{\delta} \setminus C_k^{\delta_{\ast}}}
            \mathbb{P}(d\mu dx)
    \Big)^{1/p}
\\
\numberthis
\label{eq:slope_control}
    & \le
    \max_{k\in [K]}\,
    \|f(\mu^k, x^k)\|_{KR}\,
    \Big(
        \sum_{k\in [K]}\, 
        \int_{(\mu,x) \in \mathcal{K}^{\delta_{\ast}}}\,
            I_{C_k^{\delta} \setminus C_k^{\delta_{\ast}}}
            \,
            \mathbb{P}(d\mu dx)
    \Big)^{1/p}
\\
\numberthis
\label{eq:partition}
    & =
    \max_{k\in [K]}\,
    \|f(\mu^k, x^k)\|_{KR}\,
    \Big(
        \int_{(\mu,x) \in \mathcal{K}^{\delta_{\ast}}}\,
            1
            \,
            \mathbb{P}(d\mu dx)
    \Big)^{1/p}
\\
\numberthis
\label{eq:tailMass_control}
    & \lesssim
    \max_{k\in [K]}\,
    \|f(\mu^k, x^k)\|_{KR}
    \,
    (\delta^q-\delta_{\ast}^q)^{1/p}
\end{align*}
where~\eqref{eq:definitions_of_PWC_and_PWLinear_Approxmiators} held by the definitions of $f_{\delta}$ and $\hat{f}_{\delta}$ in Lemmata~\ref{lem:PWC_Approximation}, and~\eqref{eq:tailMass_control} held by Lemma~\ref{lem:VornoiLusin}.  

Incorporating the estimate on~\eqref{t:tail_bound_completion} on the right-hand side of~\eqref{eq:tailMass_control} into the estimate in~\eqref{eq:uniform_estimate__v2__BEGIN} yields
\allowdisplaybreaks
\begin{align*}
\numberthis
\label{eq:pre_final_tail_estimate}
    \Big(
        \int_{(\mu,x) \in \mathcal{K}^{\delta_{\ast}}}\,
            \mathcal{W}_1(f(\mu,x),\hat{f}_{\delta}(\mu,x))^p
            \,
            \mathbb{P}(d\mu dx)
    \Big)^{1/p}
\lesssim 
\left(
    \omega(\delta)
    +
    \max_{k\in [K]}\,
    \|f(\mu^k, x^k)\|_{KR}
    \,
\right) (\delta^q-\delta_{\ast}^q)^{1/p}
.
\end{align*}
This establishes (iii) and completes our proof.
\end{proof}

\section{Basic Lemmata}
\label{s:Proofs__ss:basicLemmata}




\subsection{Identity Configuration}
\label{s:Proofs__ss:basicLemmata___sss:Identity}
The identity configuration gives us access to a slightly more efficient version of the deep parallelization of~\cite[Proposition 5]{cheridito2021efficient}; the proof is effectively identical.
\begin{lemma}[Deep Parallelization]
\label{lem:parallelization}
Let $I\in \mathbb{N}_+$, $\{d_i,D_i\}_{i=1}^I\subset \mathbb{N}_+$ and, for $i\in [I]$, $\Phi_i:\mathbb{R}^{d_i}\mapsto \mathbb{R}^{D_i}$ is an MLP with activation function given in~\eqref{eq:activation}, of depth $L_i$, width $W_i$, and $\operatorname{proj}_i$ non-zero parameters.  Then, there is an MLP $\Phi^{\|}:\mathbb{R}^{\sum_{i\in [I]}\,d_i}\to \mathbb{R}^{\sum_{i\in [I]}\,D_i}$ of depth at most $\sum_{i\in [I]}\,L_i + 1$, width at most $
\sum_{i\in[I]} d_i 
+
\max_{i\in [I]} W_i^2
$, and with no more that 
$
\big(
\frac{11 I^2 \max_{i\in [I]} W_i^2}{16} -1
\big)\sum_{i\in [I]}\, \operatorname{proj}_i
$ non-zero parameters, satisfying
\[
\Phi^{\|}(x_1,\dots,x_I) = \bigoplus_{i\in [I]} \Phi_i(x_i)
\]
for all $(x_1,\dots,x_I)\in \mathbb{R}^{\sum_{i\in [I]}\,d_i}$.
\end{lemma}
\begin{proof}[Proof of Lemma~\ref{lem:parallelization}]
The result follows from~\cite[Proposition 5]{cheridito2021efficient}, mutatis mutandis, since $\sigma_{(1,1,1)}$ has the $1$-identity requirement (see~\cite[Definition 4]{cheridito2021efficient}).
\end{proof}

\subsection{{ReQU Configuration 
}}
\label{s:Proofs__ss:basicLemmata___sss:ReQU}
\begin{lemma}[Exact Implementation: Multiplication]
\label{lem:mult}
Fix $d,m\in \mathbb{N}_+$.
There exists an MLP $\phi_{\times}:\mathbb{R}^{2m}\to \mathbb{R}^m$ with activation function $\sigma=\operatorname{ReQU}$ of depth $\mathcal{O}(m)$ and width  $\mathcal{O}(m)$ such that
\[
\phi_{\times}((x,y)) = (x_iy_i)_{i=1}^m
\]
for every $x,y\in \mathbb{R}^m$.
\end{lemma}
\begin{proof}[{Proof of Lemma~\ref{lem:mult}}]
Direct consequence of \cite[Proposition 1]{furuya2024simultaneously}.
\end{proof}

Together, the $ReLU$ and $ReQU$ configurations of the activation function allow us to implement the following multi-dimensional piecewise linear bump functions.
\begin{lemma}[Pseudo-Indicator Functions]
\label{lem:bump}
For $d\in \mathbb{N}_+$, and any $0<\delta_{\ast}<\delta$ there exists an MLP $\Phi_{\sqcap:\delta_{\ast},\delta}:\mathbb{R}^d\to [0,\infty)$ with activation function~\eqref{eq:activation}, depth $\mathcal{O}(d)$ and width $\mathcal{O}(d)$, which satisfies
\[
    \Phi_{\sqcap:\delta_{\ast},\delta}(x)
    =
    \prod_{i=1}^d\,
        \phi_{\sqcap:\delta_{\ast},\delta}(x_i),
\]
for each $x\in \mathbb{R}^d$,
where $\phi_{\sqcap:\delta_{\ast},\delta}$ is defined in~\eqref{eq:PW_Linear_Bump}.

In particular, consider the complement of the annulus of ``width'' $\delta-\delta_{\ast}$ given by $\mathbb{R}^d_{\delta_{\ast},\delta}\eqdef [\pm \delta_{\ast}]\bigcup (\mathbb{R}^d\setminus [\pm\delta]^d)$.  When restricted to $\mathbb{R}^d_{\delta_{\ast},\delta}$, $\Phi_{\sqcap:\delta_{\ast},\delta}$ implements the indicator function
\begin{equation}
\label{eq:pseudo_indicator}
    \Phi_{\sqcap:\delta_{\ast},\delta}|_{\mathbb{R}^d_{\delta_{\ast},\delta}}
    =
    I_{[\pm \delta_{\ast}]^d}.
\end{equation}
\end{lemma}
\begin{proof}[{Proof of Lemma~\ref{lem:bump}}]
Let $d$
By~\cite[Lemma 3.4]{ZuoweiHaizhaoZhang_2022_JMPA} there exists a ReLU MLP $\phi_{\sqcap:\delta_{\ast},\delta}:\mathbb{R}\to \mathbb{R}$ with width $5$ and depth $2$ implementing the following piecewise linear function with $4$ breakpoints (points at which it is non-differentiable)
\[
        \phi_{\sqcap:\delta_{\ast},\delta}(t) 
    = 
        \begin{cases}
            1 & |t|\le \delta_{\ast} \\
            \frac{1}{\delta_{\ast}-\delta}(|t|-\delta) & \delta_{\ast}\le t\le \delta \\
            0 & |t|>\delta
        \end{cases},
\]
for all $t\in \mathbb{R}$.  

By~\cite[Lemma 5.3]{petersen2024mathematical}, there exists a ReLU MLP $\varphi:\mathbb{R}^d\to \mathbb{R}^d$ of width at-most $10d$ and depth $3$ implementing the $d$-fold parallelization of the above ReLU MLP (applied coordinate-wise); i.e.
\[
    \varphi(x) = \bigoplus_{i=1}^d\,\phi_{\sqcap:\delta_{\ast},\delta}(x_i),
\]
for each $x\in \mathbb{R}^d$.  By Lemma~\ref{lem:mult}, the desired MLP with activation function~\eqref{eq:activation}, is given by 
$\Phi_{\sqcap:\delta_{\ast},\delta}
\eqdef 
\phi_{\times}\circ \varphi
$.  By construction its depth at $\mathcal{O}(d)$ and its width its $\mathcal{O}(d)$.
The equality in~\eqref{eq:pseudo_indicator} is now obvious.
\end{proof}

\begin{lemma}[{(Exact) Implementation of the $\ell^2_m$ Norm by a ``small'' MLP}]
\label{eq:norm_implementation__l2}
Let $F\in \mathbb{N}_+$.  There exists an MLP $\phi_{\ell^2}:\mathbb{R}^F\to \mathbb{R}$ with trainable activation function as in~\eqref{eq:activation} such that: for each $x\in \mathbb{R}^F$
\[
\phi_{\ell^2}(x)=\|x\|_2^2
.
\]
Moreover, $\phi_{\ell^2}$ has depth $1$ and width $2F$.
\end{lemma}
\begin{proof}[{Proof of Lemma~\ref{eq:norm_implementation__l2}}]
The network $\phi_{\ell^2}(x)\eqdef \operatorname{ReQU}(x)+\operatorname{ReQU}(-x)$ has depth $1$, width $2\,F$, and satisfies $\phi(x)=\|x\|_2^2$ for all $x\in \mathbb{R}^F$.
\end{proof}

More generally, we have the following \textit{quoted result}.

\begin{proposition}[{Exact Implementation of Polynomials by ReQU-ResNets}]
\label{prop:ReQU_ResNetImplementation}
Let $d,k\in \mathbb{N}$ with $d>0$, $\alpha^1,\dots,\alpha^k\in \mathbb{N}_+^d$ be multi-indices, $p\in \mathbb{R}[x_1,\dots,x_d]$ be a polynomial function on $\mathbb{R}^d$ with representation $\sum_{i=1}^k\, C_i\prod_{j=1}^d\, x_j^{\alpha^i_j} + b$, where $C_1,\dots,C_d,b\in \mathbb{R}$.  There is a ReQU-ResNet $\phi:\mathbb{R}^d\to \mathbb{R}$ with width at-most $\mathcal{O}(dk+\sum_{i=1}^k\, |\alpha^i|)$ and depth $
\mathcal{O}\big(dk+\sum_{i=1}^k\,|\alpha^i|\big)
$ exactly implementing $p$ on $\mathbb{R}^d$; i.e.\
\[
    \phi(x) = p(x),
\]
for all $x\in \mathbb{R}^d$.
\end{proposition}
See \cite[Proposition 1]{furuya2024simultaneously}.

\subsection{{ReLU Configuration 
}}

\begin{lemma}[{(Exact) Implementation of the $\ell^1_F$ Norm by a ``small'' ReLU MLP}]
\label{lem:l1d_distanceImplementation}
Let $F\in \mathbb{N}_+$.  There exists a ReLU MLP $\Phi_{F:1}:\mathbb{R}^F\to [0,\infty)$ of depth $F$ and width $\max\{F,F^2+\max\{2-F,0\}\}=F^2+\max\{2-F,0\}$ satisfying: for each $x\in \mathbb{R}^F$
\[
    \Phi_{F:1}(x) = \|x\|_{\ell^1_F}
    \eqdef 
    \sum_{f=1}^F\, |x_f|
.
\]
\end{lemma}
\begin{proof}[{Proof of Lemma~\ref{lem:l1d_distanceImplementation}}]
Recall that the absolute value function can be implemented by a ReLU MLP $\phi_0:\mathbb{R}\to\mathbb{R}$ of depth $1$ and width $2$, since: for each $x\in \mathbb{R}$ we have
\[
|x|= \operatorname{ReLU}(x) - \operatorname{ReLU}(-x)
=
\begin{pmatrix}
    1\\
    -1
\end{pmatrix}
\operatorname{ReLU}\bullet
\big((1,-1)x+ 0\big) + 0
\eqdef \phi_0(x)
.
\]
For each $i=1,\dots,F$, define the $1\times F$ matrix $A_i$ by $(A_i)_j=I_{i=j}$, for $j=1,\dots,F$.  Since the composition of affine maps is again an affine map, then observe that, for each $i=1,\dots,F$, the map
\[
\phi_i\eqdef \phi_0(A_i\cdot):\mathbb{R}^F\to [0,\infty),
\]
is a $\operatorname{ReLU}$ MLP of depth $1$ and width $\max\{2,F\}$.  Furthermore, for each $i=1,\dots,F$ and each $x\in \mathbb{R}^F$, we have that
\[
\phi_i(x) = |x_i|
.
\]
Next, consider the map $\tilde{\Phi}:\mathbb{R}^F\to [0,\infty)$ defined for each $x\in \mathbb{R}^F$ by
\begin{equation}
\label{eq:l1dist}
    \tilde{\Phi}(x)
    \eqdef 
    \sum_{i=1}^F\, \phi_i(x)\,e_i,
\end{equation}
where $\{e_i\}_{i=1}^F$ is the standard basis of $\mathbb{R}^F$.  
Since the $\operatorname{ReLU}$ activation function satisfies the $2$-identity requirement, see~\cite[Definition 4]{FlorianHighDimensional2021} and since each $\phi_i$ is a $\operatorname{ReLU}$ MLP of depth $1$ and width $\max\{2,F\}$, then we may apply \cite[Proposition 5]{FlorianHighDimensional2021} to deduce that $\Phi$ can be implemented by a $\operatorname{ReLU}$ MLP $\Phi=\tilde{\Phi}$, i.e.~$\tilde{\Phi}=\Phi$ (on all of $\mathbb{R}^F$), and that $\Phi$ has 
depth $F\cdot 1=F$ and width $F(F-1)+\max\{2,F\}= F^2 +\max\{2-F,0\}
$.
Next, consider the function $\Phi_1:\mathbb{R}^F\to [0,1]$ defined for each $x\in \mathbb{R}^F$ by
\[
    \Phi_1(x)\eqdef \mathbf{1}^{\top}\,\Phi(x),
\]
where $\mathbf{1}\in \mathbb{R}^F$ has all its components equal to $1$.  Again, using the fact that the composition of affine maps is again an affine map, we see that $\Phi_1$ is a $\operatorname{ReLU}$ MLP of depth $F$ and width $\max\{F,F^2+\max\{2-F,0\}\}=F^2+\max\{2-F,0\} (\in \mathcal{O}(F^2))
$.  Furthermore, for each $x\in \mathbb{R}^F$ we have that
\[
        \Phi_1(x) 
    = 
        \mathbf{1}^{\top} \sum_{i=1}^F\,\phi_i(x)\,e_i 
    = 
        \sum_{i=1}^F\, \phi_i(x)
    = 
        \sum_{i=1}^F\, |x_i|
    =
        \|x\|_1.
\]
We conclude that $\Phi_1\eqdef\|\cdot\|_1$.  Relabeling $\Phi_1$ as $\Phi_{F:1}$ yields the conclusion.
\end{proof}
This brings us to a key result, which is that ReLU MLPs can \textit{exactly} implement the $1$-Wasserstein distance on $\mathcal{P}_N(\mathbb{R}^d)$.

We begin with a special case of our general result, Lemma~\ref{lem:implementation_W1}, showing that the $1$-Wasserstein distance can be exactly computed by a deep ReLU MLP on $\mathcal{P}_{C,N}(\mathcal{X})$.
\begin{lemma}[{Exact Implementation of $1$-Wasserstein Distance on $\mathcal{P}_N(\mathbb{R}^d)$}]
\label{lem:implementation_W1__specialCase}
Let $N,d\in \mathbb{N}_+$ with $N>1$.  There exist a ReLU MLP $\Phi_{W^1}:\mathbb{R}^{N\times F}\times \mathbb{R}^{N\times F}\to [0,\infty)$ such that: for every $\mu,\nu\in \mathcal{P}_N(\mathbb{R}^d)$
\[
    \mathcal{W}_1(\mu,\nu) = \Phi_{W^1}(\mathbf{X}^{\mu},\mathbf{Y}^{\nu}),
\]
where $\mu = \sum_{n=1}^N\, \frac1{N}\, \delta_{\mathbf{X}^{\mu}_n}$ and $\nu = \sum_{n=1}^N\, \frac1{N}\, \delta_{\mathbf{Y}^{\nu}_n}$; furthermore, and this holds independently of the chosen representation of $\mu$ and $\nu$ as $N\times d$ matrices.
\hfill\\
Moreover, $\Phi_{W^1}$ has depth at-most $\mathcal{O}(N(d+\log_2(N)))$ and width $\mathcal{O}(N! N^2d)$.
\end{lemma}


\begin{proof}[{Proof of Lemma~\ref{lem:implementation_W1__specialCase}}]

\textbf{Step 1 - Combinatorial Representation of $1$-Wasserstein Distance:}\\
Since optimal couplings between empirical measures are representable by permutations of the points of mass of each measure, see e.g.\cite[Proposition 2.1]{peyre2019computational}, we have that: for every $\mu,\nu\in\mathcal{P}_N(\mathbb{R}^d)$
\begin{equation}
\label{eq:simple_Wasserstein1_formula}
        \mathcal{W}_1\big(\mu,\nu\big)
    =
        \min\limits_{\pi\in S^N}\,
            \frac{1}{N}\,
                \sum_{n=1}^N\,
                    \|x_n - y_{\pi(n)}\|_1,
\end{equation}
where $\mu = \sum_{n=1}^N\, \frac1{N}\, \delta_{x_n}$, $\nu = \sum_{n=1}^N\, \frac1{N}\,\delta_{y_n}$, $X\eqdef (x_n)_{n=1}^N$, and $Y\eqdef (y_n)_{n=1}^N$ (we have  reduced notation to simplify legibility).  
This may be expressed equivalently as
\begin{equation}
\label{eq:simple_Wasserstein1_formula__MatV}
        \mathcal{W}_1\big(\mu,\nu\big)
    =
        \min\limits_{\Pi\in \mathbf{S}^N}\,
            \frac{1}{N}\,
                \sum_{n=1}^N\,
                    \|P_nX - P_n\Pi Y_n\|_1,
\end{equation}
where, for each $n\in [N]$, $P_n:\mathbb{R}^{N\times d}\to \mathbb{R}^d$ is the projector matrix mapping an $N\times d$ matrix to its $n^{th}$ column vector.

\textbf{Step 2 - Implementation as ReLU MLP:}
\hfill\\
For each $n\in [N]$ and $\Pi\in \mathbf{S}^N$ consider the map $\mathbb{R}^{2(N\times d)}
\in (X,Y)\mapsto \|P_n X - P_n\Pi Y\|_1 \in [0,\infty)$.  
Since the map $\mathbb{R}^{2(N\times d)}(X,Y)\ni (X,Y)\mapsto P_nX - P_n\Pi_Y\in \mathbb{R}^{N\times d}$ is affine, and the composition of affine maps is again affine, then Lemma~\ref{lem:l1d_distanceImplementation}, there exists a ReLU MLP $\Phi_{n,\Pi}:\mathbb{R}^{2(N\times d)}\to \mathbb{R}$ such that
\begin{equation}
\label{eq:implementation_permuted_norm}
\Phi_{N,\Pi}(X,Y) = \|P_n X - P_n\Pi Y\|_1,
\end{equation}
for all $(X,Y)\in \mathbb{R}^{2(N\times d)}$.  Moreover, each $\Phi_{n,\Pi}$ has depth $Nd$ and width $Nd + \max\{2-Nd,0\}= Nd$.

By parallelization, see~\cite[Lemma 5.3]{petersen2024mathematical}, for each $\Pi\in \mathbf{S}^N$, there exists a ReLU MLP $\Phi_{\Pi}^{\|}:\mathbb{R}^{2(N\times d)} \mapsto \mathbb{R}^N$ satisfying
\begin{equation}
\label{eq:implementation_permuted_norm__paralelized_over_N}
        \Phi_{\Pi}^{\|}
            (X,Y)
    =
        \bigoplus_{n=1}^N\,
            \Phi_{N,\Pi}
                (X,Y)
    = 
        \bigoplus_{n=1}^N\, 
            \|P_n X - P_n\Pi Y\|_1,
\end{equation}
for all $(X,Y)\in \mathbb{R}^{2(N\times Y)}$.  Furthermore, $\Phi_{\Pi}^{\|}$ has depth $Nd$ and width $2N(Nd)=2N^2d$.  
Let $\mathbf{1}_N \in\mathbb{R}^N$ have all components equal to $1/N$.  Since the composition of affine maps is again affine then, for each $\Pi\in \mathbf{S}^N$, the function $\Phi_{\Pi}\eqdef \mathbf{1}_N^{\top}\, \Phi_{\Pi}^{\|}:\mathbb{R}^{2(N\times d)}\to \mathbb{R}$ given by
\begin{equation}
\label{eq:implementation_permuted_norm__averaged_over_N}
        \Phi_{\Pi}(X,Y)
    =
        \sum_{n=1}^N\,
            \frac1{N}\, \Phi_{N,\Pi}
                (X,Y)
    = 
        \sum_{n=1}^N\, 
            \frac1{N}\,
                \|P_n X - P_n\Pi Y\|_1,
\end{equation}
where $(X,Y)\in \mathbb{R}^{2(N\times d)}$ is itself a ReLU MLP of depth and width equal to that of $\Phi_{\Pi}^{\|}$; i.e.\ its depth is depth $Nd$ and its width is $2N(Nd)=2N^2d$.

Parallelizing each $\Phi_{\Pi}$, again using~\cite[Lemma 5.3]{petersen2024mathematical}, there is a ReLU MLP $\Phi^{\|}:\mathbb{R}^{2(N\times d)}\to \mathbb{R}^{N!}$ (since the cardinality of $\mathbf{S}^N$ is $N!$) satisfying
\begin{equation}
\label{eq:implementaiton_parallelization_over_all_permutations}
        \Phi^{\|}(X,Y)
    =
        \bigoplus_{\Pi\in \mathbf{S}^N}\,
            \Phi_{\Pi}(X,Y)
    =
        \bigoplus_{\Pi\in \mathbf{S}^N}\,
            \Biggl(
                \sum_{n=1}^N\, 
                    \frac1{N}\,
                        \|P_n X - P_n\Pi Y\|_1
            \Biggr)
.
\end{equation}
Furthermore, $\Phi^{\|}$ has depth $F$ and width $(2N!)(2N^2d)$.

By~\cite[Lemma 5.11]{petersen2024mathematical} there is a ReLU MLP $\Phi_{\text{min}}:\mathbb{R}^{N!}\to \mathbb{R}$ with width at most $3(N!)$ and depth at most $\lceil \log_2(N!)\rceil$ implementing
\[
    \Phi_{\text{min}}(Z) = \min_{n\in [N]}\, Z_n,
\]
for each $Z\in \mathbb{R}^{N!}$.  By~\cite[Lemma 5.2]{petersen2024mathematical}, the composite function $\Phi_{W^1}\eqdef \Phi_{\text{min}}\circ \Phi^{\|}$ is itself a ReLU MLP of width at most $
2N^2 d
+
3(N!)
$
and depth at most $
Nd
+
\lceil \log_2(N!)\rceil 
$.  Using Stirling's estimate, we have that the depth of $\Phi_{W^1}$ is at most $\mathcal{O}\big(N(d+\log(N))\big)$.  Moreover, by construction 
\begin{equation}
\label{eq:implemented_W1__done}
        \Phi_{W^1}(X,Y) 
    = 
        \min_{\Pi\in \mathbf{S}^N}\,
            \Biggl(
                \sum_{n=1}^N\, 
                    \frac1{N}\,
                        \|P_n X - P_n\Pi Y\|_1
            \Biggr),
\end{equation}
for all $(X,Y)\in \mathbb{R}^{2(N\times d)}$.  
Noting that the left-hand side of~\eqref{eq:implemented_W1__done} equals to the combinatorial form of the $W_1$ distance on $\mathcal{P}_N(\mathbb{R}^d)$ in~\eqref{eq:simple_Wasserstein1_formula__MatV} completes the proof.
\end{proof}

\subsection{Exactly Implementing the Wasserstein $1$-Distance}
\label{s:Exact_W1}

\begin{proposition}[{Implementation of Wasserstein Distance on $\mathcal{P}_{C,N}(\mathbb{R}^d)$}]
\label{prop:Computation_W1__relative_verison}
Let $C,N\in \mathbb{N}_+$, $\mathcal{X}\subseteq \mathbb{R}^d$ with at-least $N$ distinct points, and $w,v\in \Delta_{C,N}$.  Let 
\begin{equation}
\label{eq:relative_contextualized_probabilityMeasures}
\mathcal{P}_{w,v}(\mathcal{X})\eqdef 
\Big\{
        (\mu,\nu)\in \mathcal{P}_{C,N}(\mathcal{X})^2
    :
        \,
        \mu = \sum_{n=1}^N\, w_n\,\delta_{x_n}
        ,\,
        \nu = \sum_{n=1}^N\, v_n\,\delta_{y_n}
\Big\}
.
\end{equation}
There exists an MLP $\phi_{\mathcal{W}_1:w,v}:\mathbb{R}^{N\times d}\times \mathbb{R}^{N\times d}\to [0,\infty)$ with activation function~\eqref{eq:activation} such that
\[
        \phi_{\mathcal{W}_1:w,v}(\mathbf{X},\mathbf{Y})
    =
        \mathcal{W}_1\biggl(
                \sum_{n=1}^N\, w_n\,\delta_{x_n}
            ,
                \sum_{n=1}^N\, v_n\,\delta_{y_n}
        \biggr),
\]
for each $\big(\sum_{n=1}^N\, w_n\,\delta_{x_n},\sum_{n=1}^N\, v_n\,\delta_{y_n}\big)\in \mathcal{P}_{w,v}(\mathcal{X})$; 
where $\mathbf{X}$ (resp.\ $\mathbf{Y}$) is \textit{any} matrix with rows given by $x_1,\dots,x_N$ (resp. $y_1,\dots,y_N$) (independently of the choice of row ordering).
\hfill\\
Moreover, $\phi_{\mathcal{W}_1:w,v}$ has depth and width at-most $\mathcal{O}(N^{N-1})$ and at-least $\Omega(N!)$.
\end{proposition}
\begin{proof}[{Proof of Proposition~\ref{prop:Computation_W1__relative_verison}}]
Fix $\big(\sum_{n=1}^N\, w_n\,\delta_{x_n},\sum_{n=1}^N\, v_n\,\delta_{y_n}\big)\in \mathcal{P}_{w,v}(\mathcal{X})$.  
As shown in~\cite[Equations (2.10)-(2.11)]{Gabriel_2019computational}, 
\begin{equation}
\label{eq:Linear_Program__W1}
    \mathcal{W}_1\big(
    \mu
    ,
    \nu
    \big)
=
    \min_{P\in U(w,v)}\,
        \langle C, P\rangle,
\end{equation}
where $C \eqdef \big(\|x_n-y_m\|_1\big)_{n,m=1}^N$ and where is the transport polytope, defined in~\cite[Equations (2.10)]{Gabriel_2019computational} associated to the \textit{weights} $w,v\in \Delta_{C,N}$ (importantly, we note that $U(w,v)$ does not depend on the points $\{x_n,y_n\}_{n=1}^N$).
Since $U(w,v)$ is a convex polytope, and the right-hand side of~\eqref{eq:Linear_Program__W1} is a linear program with convex polytopal constraints, then the set of optimizer(s) of the right-hand side of~\eqref{eq:Linear_Program__W1} belong to the set of \textit{extremal points} of $U(w,v)$, which we denote by  $U^{\star}(w,v)$; see e.g.~\cite[Theorem 2.7]{bertsimas1997introduction}.  
Therefore,~\eqref{eq:Linear_Program__W1} can be re-extressed as
\begin{equation}
\label{eq:Linear_Program__W1___finitized}
    \mathcal{W}_1\big(
    \mu
    ,
    \nu
    \big)
=
    \min_{k=1,\dots,K}\,
        \langle 
        (\big|
            x_i - y_j
        \big|)_{i,j=1}^N
        , \operatorname{proj}_k\rangle,
\end{equation}
where $U^{\star}(w,v)=\{\operatorname{proj}_k\}_{k=1}^K$ for some matrices $\operatorname{proj}_1,\dots,\operatorname{proj}_K\in \mathbb{R}^{N^2}$, for some $K\in \mathbb{N}_+$.
By \cite[Theorem 8.1.5 and Theorem 8.1.6]{Brualdi_CombinatorialMatrixClasses_Book_2006}, we have that the two-sided estimates
\begin{equation}
\label{eq:bound_on_number_extremalpoints}
        N!
    \le
        K
    \le
        N^{N-1}
.
\end{equation}
By Lemma~\ref{lem:l1d_distanceImplementation}, there exists a ReLU MLP $\Phi_{d:1}:\mathbb{R}^d\to [0,\infty)$ of depth $d$ and width $\mathcal{O}(d^2)$ satisfying $\Phi_{d:1}=\|\cdot\|_1$.
By Proposition~\ref{prop:ReQU_ResNetImplementation}, there exists a ReQU MLP $\phi_{\langle,\rangle}$ with $\mathcal{O}(N^2)$ depth and width implementing the inner product on $\mathbb{R}^{N^2}$.  
By~\citep[Lemma 5.11]{petersen2024mathematical}, there exists an MLP $\phi_{\min}$ of depth $\mathcal{O}(\log(N))$ and width $\mathcal{O}(K)$ implementing the componentwise minimum function on $\mathbb{R}^{K}$.  
Together, these observations imply that~\ref{eq:Linear_Program__W1___finitized} can be rewritten as
\begin{equation}
\label{eq:Linear_Program__W1___partial_MLPization}
    \mathcal{W}_1\big(
        \mu
    ,
        \nu
    \big)
=
    \phi_{\min}\biggl(
        \phi_{\langle,\rangle}\Big(
                \oplus_{i,j=1}^N
                \,
                    \Phi_{d:1}(\operatorname{proj}_iX - y_j)        
            , 
                \operatorname{proj}_k
        \Big)
    _{k=1}^K
    \biggr),
\end{equation}
where, for $i=1,\dots,N$, $\operatorname{proj}_i:\mathbb{R}^{N\times N}\to \mathbb{R}^N$ is the canonical (linear) projection onto the $i^{th}$ row%
; whence $\Phi_{d:1}(\operatorname{proj}_iX - y_j)$ is an MLP of depth and width equal to that of $\Phi_{d:1}$ (here, we have used the fact that the composition of affine functions is again affine).

Applying the deep parallelization lemma, Lemma~\ref{lem:parallelization}, we find that there exists a ReLU MLP $\Phi^{\|}:\mathbb{R}^{N^2}\to \mathbb{R}^{N^2}$ implementing $ \oplus_{i,j=1}^N
                \,
                    \Phi_{d:1}(\operatorname{proj}_i\cdot - y_j)$; moreover,
$\Phi^{\|}$ has depth at-most $\mathcal{O}(N^4)$ and width at-most $\mathcal{O}(N^2)$.  Therefore,~\eqref{eq:Linear_Program__W1___partial_MLPization} reduces to
\begin{equation}
\label{eq:Linear_Program__W1___partial_MLPization__parallelized}
    \mathcal{W}_1\big(
        \mu
    ,
        \nu
    \big)
=
    \phi_{\min}\biggl(
        \phi_{\langle,\rangle}\Big(
                \Phi^{\|}(X,Y)     
            , 
                \operatorname{proj}_k
        \Big)
    _{k=1}^K
    \biggr)
.
\end{equation}
Therefore, $\mathcal{W}_1$ can be implemented by an MLP with activation function~\eqref{eq:activation}, depth at-most $\mathcal{O}(N^4 + K)$ and width at-most $\mathcal{O}(N^2+K)$.  The conclusion follows upon appealing to the estimate of $K$ in~\eqref{eq:bound_on_number_extremalpoints}.
\end{proof}

Upon stitching together every possible combination of pairs of contextualized weights in $\Delta_{C,N}$, Proposition~\ref{prop:Computation_W1__relative_verison} yields Lemma~\ref{lem:implementation_W1}, which guarantees that the $1$-Wasserstein distance can be \textit{exactly} computed for any pair of measures in $\mathcal{P}_{C,N}(\mathcal{X})$, for appropriate $\mathcal{X}$, be a fixed ReLU MLP.

\begin{lemma}[{Implementation of Wasserstein Distance - All Contextual Measures 
}]
\label{lem:implementation_W1}
Let $C,N\in \mathbb{N}_+$, $\mathcal{X}\subseteq \mathbb{R}^d$ with at-least $N$ distinct points.  
There exists an MLP $\phi_{\mathcal{W}_1}:\mathbb{R}^{N\times d} \times \mathbb{R}^N \times \mathbb{R}^{N\times d}\times \mathbb{R}^N \to \mathbb{R}$ with activation function~\eqref{eq:activation} such that
\[
\phi_{\mathcal{W}_1}(\mathbf{X}, w, \mathbf{Y}, v)
    =
        \mathcal{W}_1\biggl(
                \sum_{n=1}^N\, w_n\,\delta_{x_n}
            ,
                \sum_{n=1}^N\, v_n\,\delta_{y_n}
        \biggr),
\]
for each $\sum_{n=1}^N\, w_n\,\delta_{x_n},\sum_{n=1}^N\, v_n\,\delta_{y_n}\in \mathcal{P}_{C,N}(\mathcal{X})$; 
where $\mathbf{X}$ (resp.\ $\mathbf{Y}$) is the matrix with rows given by $x_1,\dots,x_N$ (resp. $y_1,\dots,y_N$).
Moreover, the depth and width of $\phi_{\mathcal{W}_1}$ are at most $\mathcal{O}(N^{N+2C-3})$ and at-least $\Omega(N! \, N^{2C-2})$.
\end{lemma}
\begin{proof}[{Proof of Lemma~\ref{lem:implementation_W1}}]
\textit{Our of notational convenience, we write $(\mathbf{X},w)$ for $(x_n,w_n)_{n=1}^N$; thereby picking out an arbitrary representative of $(\mathbf{X},w)$ from its equivalence class and representing it with an element of $\mathbb{R}^{Nd}\times \Delta_N$.  Our construction is independent of this choice, and thus, there is no ambiguity in making this choice.}

Note that, by a simple integer composition observation, the number of elements of $\Delta_{C,N}$ is exactly 
\begin{equation}
\label{eq:counting_DeltaCN}
    \#\Delta_{C,N}
    =
    \binom{C+N-1}{N-1}
.
\end{equation}
Therefore, the number of possible (ordered) pairs of weights in $\Delta_{C,N}^2$ is 
\begin{equation}
\label{eq:counting_DeltaCN__pairs}
\begin{aligned}
        \#\Delta_{C,N}^2
    & \eqdef 
        N^{\star}
\\
    & =
        \Big(
            \binom{C+N-1}{N-1}
        \Big)^2
\\
    & =
        \exp\Big(
            \log\big(
                \prod_{k=1}^{N-1}
                \,
                1 + \frac{C-1}{k}
            \big)
        \Big)^2
\\
    & =
        \exp\Big(
            \sum_{k=1}^{N-1}
            \log\big(
                1 + \frac{C-1}{k}
            \big)
        \Big)^2
\\
    & \in 
        \mathcal{O}\Big(
            N^{2C-2}
        \Big) 
.
\end{aligned}
\end{equation}
Applying Proposition~\ref{prop:Computation_W1__relative_verison} $N^{\star}$ times, once for each ordered pair of contextualized weights $(w,v)\in \Delta_{C,N}^2$ yields a family of MLPs 
$
\big\{
    \phi_{\mathcal{W}_1:w,v}
\big\}_{w,v\in \Delta_{C,N}^2}
$, where $\phi_{\mathcal{W}_1:w,v}=\mathcal{W}_1|_{\mathcal{P}_{w,v}(\mathbb{R}^d)}$.

Since the minimum $\ell^1$ distance between \textit{distinct} weights $w,\tilde{w}\in \Delta_{C,N}$ is at least $1/C$, then consider the piecewise linear function with exactly two break points at $0$ and at $1/(2C)$ which takes value $1$ on $(-\infty,0]$ and takes value $0$ on $[1/(2C),\infty)$.  Clearly, this function can be implemented by a ReLU MLP $\phi_{\lceil \rceil}:\mathbb{R}\to [0,1]$ of $\mathcal{O}(1)$ depth and width.
By Lemma~\ref{lem:l1d_distanceImplementation} there exists a ReLU MLP $\Phi_{N:1}:\mathbb{R}^N\to [0,\infty)$ of depth $N$ and width $\mathcal{O}(N^2)$ implementing the $\ell^1$ norm on $\mathbb{R}^N$.  Therefore, the ReLU MLP
$\Phi_{w,v}:\mathbb{R}^{2N}\to [0,1]$ given for each $(\tilde{w},\tilde{v})\in \mathbb{R}^{2N}$ by
\begin{equation}
\label{eq:ReLUMLPS_wv_identifier}
    \Phi_{\tilde{w},\tilde{v}}(w,v)
\eqdef 
    \phi_{\lceil \rceil}
    \Big(
        (1,1)^{\top}\big(
            \Phi_{N:1}(w-\tilde{w}), \Phi_{N:1}(v-\tilde{v})
        \big)
    \Big),
\end{equation}
for all $w,v\in \Delta_{C,N}$; moreover, by construction $\Phi_{w,v}$ has depth $\mathcal{O}(N)$ and width $\mathcal{O}(N^2)$; where we have used the deep parallelization Lemma~\ref{lem:parallelization}.

Using Lemma~\ref{lem:mult}, for each $(w,v)\in \Delta_{C,N}^2$ there exists an MLP $\Psi_{w,v}:\mathbb{R}^{2N}\times \mathbb{2\,N^2}\to \mathbb{R}$ with activation function~\eqref{eq:activation} satisfying 
\begin{equation}
\label{eq:ReLUMLPS_wv_identifier-2}
        \Psi_{\tilde{w},\tilde{v}}
        \Big(
        \mathcal{X}, w, \mathcal{Y}, v
        \Big)
    \eqdef
        \phi_{\times}\big(
                \Phi_{\tilde{w},\tilde{v}}(w,v)
            ,
                \phi_{\mathcal{W}_1:\tilde{w},\tilde{v}}((x_n)_{n=1}^N,(y_n)_{n=1}^N)
        \big),
\end{equation}
where $\big(w_n,x_n\big)_{n=1}^N,\big(v_n,y_n\big)_{n=1}^N\in \mathbb{R}^{N+N^2}$ are such that $w,v\in \Delta_{C,N}$.
Applying the Parallelization Lemma, Lemma~\ref{lem:parallelization}, $N^{\star}$ times, any post-composing the result with the linear map $\mathbb{R}^{N^{\star}}\ni x \to \sum_{n=1}^{N^{\star}} x_n \in \mathbb{R}$ yields the conclusion; upon recalling the estimate on $N^{\star}$ in~\eqref{eq:counting_DeltaCN__pairs}.  
We thus obtain
$$
\phi_{\mathcal{W}_1}(\mathbf{X},w,\mathbf{Y},v)
= \sum_{(\tilde{w}, \tilde{v}) \in \Delta_{C,N}} \Psi_{\tilde{w},\tilde{v}}
        \Big(
        \mathcal{X}, w, \mathcal{Y}, v
        \Big),
$$
which completes the proof of our lemma.
\end{proof}

\section{Proofs of Supporting Results}
\label{s:Proof_Extras}
\begin{proof}[{Proof of Proposition~\ref{prop:Identifcation}}]
By construction, the metric spaces $\big(\operatorname{Mat}_N^{d,N},\mathcal{W}\big)$ and $\big(\mathcal{P}_N^{N,d}(\mathcal{X}),\mathcal{W}_1\big)$ is isometric via~\eqref{eq:W1_onMat}.  Thus, without loss of generality, we work with the later up to the identification $\Phi$.
Let $\mu = \frac1{N}\,\sum_{n=1}^N\, \delta_{X_n},\nu= \frac1{N}\,\sum_{n=1}^N\, \delta_{Y_n} \in \mathcal{P}^{N:+}(\mathcal{X})$.  
By~\cite[Proposition 2.1]{peyre2019computational}, 
\allowdisplaybreaks
\begin{align*}
        \mathcal{W}_1\big(\mu,\nu\big)
    & =
        \min\limits_{\pi\in S^N}\,
            \frac{1}{N}\,
                \sum_{n=1}^N\,
                    \|X_n - Y_{\pi(n)}\|_1
\\
    &
    \ge 
        \min\limits_{\pi\in S^N}\,
            \frac{1}{N}\,
                \sum_{n=1}^N\,
                    \|X_n - Y_{\pi(n)}\|_2
\\
    &
    \ge 
        \frac1{N}\,
        \inf_{(X^{(1)},\dots,X^{(T)})}\,
        \sum_{t=1}^{T-1}
        \min\limits_{\pi\in S^N}\,
            \frac{1}{N}\,
                \sum_{n=1}^N\,
                    \|X_n^{(t)} - X_{\pi(n)}^{(t+1)}\|_2
\\
    &
    = \frac1{N}\,d_2([X],[Y])
,
\end{align*}
where, again, the infimum is taken over all (paths) sequences $X^{(1)},\dots,X^{(T)}\in \mathbb{R}^{Nd}$ with $X^{(1)}=X$ and $X^{(T)}=Y$ (here $T\in \mathbb{N}_+$ with $T\ge 2$).  Thus, the map $\psi$ is $N$-Lipschitz when its domain is equipped with the restriction of the $1$-Wasserstein metric and its codomain is equipped with the quotient metric $d_2$, yielding the first conclusion.

It remains to show that $d_2$ metrizes the quotient topology on $\operatorname{GL}(N,d)/\mathbf{S}^N$ to obtain our second conclusion.
Next, we observe that $\mathbf{S}^N$ acts on $\operatorname{GL}(N,d)$ by row-permutations, equivalently by matrix multiplication on the left.  Whence, $\mathbf{S}^N$ acts on $\operatorname{GL}(N,d)$ by isometries since, for each $\Pi\in \mathbf{S}^N$ and every $X=((X_n)_{n=1}^N)^{\top},
Y = ((Y_n)_{n=1}^N)^{\top}
\in \operatorname{GL}(N,d)$ we have
\begin{equation}
    \|\Pi\cdot X - \Pi \cdot Y\|_2^2
    =
    \sum_{n=1}^N\,
        \|X_{\pi(n)}-Y_{\pi(n)}\|^2_2
    =
    \sum_{n=1}^N\,
        \|X_n-Y_n\|^2_2
    =
    \|X-Y\|_2^2
,
\end{equation}
where $\Pi$ is the obvious linear representation of the permutation $\pi\in S^n$.  Therefore, by~\cite[Exercise 8.4 (3) - page 132]{BridsonHaefliger_1999NPCBook}, the quotient metric $d_2$ metrizes the \textit{quotient topology} on $\operatorname{GL}(N,d)/\mathbf{S}^N$; whence the map $\psi$ is continuous.  

We have shown that, the metrics $\mathcal{W}$ and $\operatorname{dist}$ are equivalent on $\operatorname{Mat}_N^{N,d}/\sim$; i.e.\ that $\Phi$ is a bi-Lipschitz surjection.  It thus follows that both spaces are homeomorphic.
\end{proof}

\end{document}